%% file: neurips_2025.tex
\documentclass{article}

\PassOptionsToPackage{numbers, sort&compress}{natbib}
\usepackage[final]{neurips_2025}

\usepackage[utf8]{inputenc} 
\usepackage[T1]{fontenc}    
\usepackage{hyperref}       
\usepackage{url}            
\usepackage{booktabs}       
\usepackage{amsfonts}       
\usepackage{nicefrac}       
\usepackage{microtype}      
\usepackage{xcolor}         
\usepackage{multirow}
\usepackage{array}
\usepackage[table]{xcolor}
\usepackage{colortbl}
\input{dfn} 

\title{Parameter-Free Hypergraph Neural Network for Few-Shot Node Classification}

\author{%
  \begin{tabular}{cccc}
    Chaewoon Bae & Doyun Choi & Jaehyun Lee & Jaemin Yoo
  \end{tabular} \\
  School of Electrical Engineering \\
  Korea Advanced Institute of Science and Technology (KAIST)\\
  \texttt{\{chaewoon.bae, doyun.choi, jaehyun.lee, jaemin\}@kaist.ac.kr} \\
}

\begin{document}

\maketitle

\input{sections/000abstract}
\input{sections/010introduction}

\input{sections/020relatedwork}
\input{sections/030methodology}
\input{sections/040experiment}
\input{sections/050discussion}

\section*{Acknowledgements}

This work was supported by the National Research Foundation of Korea (NRF) grant funded by the Korea government (MSIT) (RS-2024-00341425 and RS-2024-00406985).
This work was supported by BK21 FOUR (Connected AI Education \& Research Program for Industry and Society Innovation, KAIST EE, No. 4120200113769).
Jaemin Yoo is the corresponding author.

{\small
\bibliographystyle{plainnat}
\bibliography{reference}}

\input{checklist}
\newpage
\input{appendices/appendix}

%


\end{document}

%% file: dfn.tex
\usepackage{xspace}
\usepackage{amsthm}
\usepackage{amsmath}
\usepackage{graphicx}
\usepackage{amssymb}
\usepackage{algorithm}
\usepackage{algpseudocode}
\usepackage{svg}

\usepackage{enumitem}
\setlist[itemize]{leftmargin=*, itemsep=0pt, topsep=0pt}

\newtheorem{theorem}{Theorem}
\newtheorem{assumption}{Assumption}
\newtheorem{lemma}{Lemma}

\newtheorem{definition}{Definition}

\newcommand{\method}{ZEN\xspace}
\newcommand{\methodlong}{\underline{Z}ero-Parameter Hyp\underline{e}rgraph \underline{N}eural Network\xspace}

\newcommand{\mypara}[1]{\textbf{#1} \ \ }

\def\A{\mathbf{A}}
\def\D{\mathbf{D}}
\def\H{\mathbf{H}}
\def\I{\mathbf{I}}
\def\J{\mathbf{J}}
\def\K{\mathbf{K}}
\def\M{\mathbf{M}}
\def\P{\mathbf{P}}

\def\W{\mathbf{W}}
\def\X{\mathbf{X}}
\def\Y{\mathbf{Y}}
\def\Z{\mathbf{Z}}

\def\Dv{\mathbf{D}_\mathrm{v}}
\def\Dvh{\mathbf{D}_\mathrm{v}^{-\frac{1}{2}}}
\def\Dvf{\mathbf{D}_\mathrm{v}^{-1}}
\def\De{\mathbf{D}_\mathrm{e}}
\def\Def{\mathbf{D}_\mathrm{e}^{-1}}

\def\Dtrn{\mathbf{D}_\mathrm{trn}}
\def\Dtest{\mathbf{D}_\mathrm{test}}
\def\Yhat{\hat{\mathbf{Y}}}
\def\Ytrn{\mathbf{Y}_\mathrm{trn}}

\def\RSI{\mathrm{RSI}}

%% file: sections/000abstract.tex
\begin{abstract}
    Few-shot node classification on hypergraphs requires models that generalize from scarce labels while capturing high-order structures. Existing hypergraph neural networks (HNNs) effectively encode such structures but often suffer from overfitting and scalability issues due to complex, black-box architectures.
    In this work, we propose \method (\methodlong), a fully linear and parameter-free model that achieves both expressiveness and efficiency. Built upon a unified formulation of linearized HNNs, \method introduces a tractable closed-form solution for the weight matrix and a redundancy-aware propagation scheme to avoid iterative training and to eliminate redundant self-information.
    On 11 real-world hypergraph benchmarks, \method consistently outperforms eight baseline models in classification accuracy while achieving up to $696\times$ speedups over the fastest competitor. 
    Moreover, the decision process of \method is fully interpretable, providing insights into the characteristic of a dataset.
    Our code and datasets are fully available at \url{https://github.com/chaewoonbae/ZEN}.

\end{abstract}


%% file: sections/010introduction.tex
\section{Introduction}

Many real-world datasets are naturally modeled as hypergraphs, which generalize ordinary graphs by capturing higher-order interactions involving more than two nodes \cite{lee2022mining, kim2024survey, antelmi2023survey}. Unlike traditional graphs that represent pairwise relationships, hypergraphs allow for the modeling of multilateral connections, making them particularly suitable for complex data structures. Hypergraphs are especially useful in representing multifaceted relationships across a range of domains, such as document-word associations, gene-disease correlations, and user-item interactions \cite{NEURIPS2019_1efa39bc,
NIPS2013_8a3363ab,
Feng_You_Zhang_Ji_Gao_2019,
pmlr-v70-zhang17d}.


Few-shot node classification—predicting node labels with only a handful of annotated examples—is a fundamental yet challenging task, especially in the context of hypergraphs \cite{10.1145/3357384.3358106, 10.1145/3534678.3539265}. The complexity of higher-order structures, combined with the scarcity of labeled data, makes it difficult to design models that are both generalizable and efficient. Although a number of hypergraph neural networks (HNNs) have been proposed to effectively capture high-order relationships between nodes, many of them rely on complex, nonlinear architectures with numerous parameters. Such models often suffer from overfitting and poor scalability in few-shot scenarios \cite{yu2024surveyfewshotlearninggraphs, tang2024simplifying}.

In contrast, linear models offer strong generalization, low complexity, and are particularly effective in few-shot scenarios \cite{lee2019metalearningdifferentiableconvexoptimization,
du2021fewshotlearninglearningrepresentation, wang2021dissecting, li2022g}. Despite these advantages, existing work on HNNs has largely overlooked fully linear formulations, likely due to the perception that linear models cannot sufficiently capture the structural richness of hypergraphs. This motivates us to explore if a carefully designed linear HNN can overcome these limitations while preserving both expressiveness and efficiency.

In this paper, we propose \method (\methodlong), a model designed to capture high-order relationships while maintaining strong generalization ability and scalability. We begin by linearizing existing HNN models, yielding a general form characterized by a single propagation matrix and a single weight matrix. This formulation enables us to design two optimization strategies tailored to this structure: \emph{tractable closed-form solution} (TCS) for the weight matrix and \emph{redundancy-aware propagation} (RAP). TCS allows us to avoid burdensome iterative training by deriving an efficient closed-form solution, while RAP eliminates redundant self-information across multi-hop propagation.
\method achieves the best average rank in classification performance across 11 real-world hypergraphs, while also exhibiting exceptional computational efficiency.

To summarize, our main contributions are as follows:
\begin{itemize}
    \item We present a general form of linearizing representative hypergraph neural networks, and exhibit similarities and differences between different HNNs excluding their nonlinearity. To the best of our knowledge, this is the first comprehensive study on linear HNNs.

    \item We derive a tractable closed-form approximation for the weight matrix, the only parameter in a linear HNN. 
    Our solution eliminates the need for computing matrix pseudoinverses, significantly improving computational efficiency while maintaining high classification accuracy.

    \item We introduce \emph{redundancy-aware propagation} that effectively reduces the structural overlap across multi-hop neighborhoods when building the propagation matrix. This leads to more efficient and expressive information aggregation, especially uesful for complex real hypergraphs.

    \item We conduct comprehensive experiments on {11 real-world hypergraphs} and demonstrate that our method outperforms existing HNNs in few-shot classification while showing exceptional scalability; our \method is up to $696\times$ faster than the fastest competitor.
    In addition, we conduct a case study that shows the interpretability of \method arising from its linear decision process.
\end{itemize}

%% file: sections/020relatedwork.tex
\section{Problem definition and related work}
\label{sec:related-work}

\subsection{Problem definition}

We consider a hypergraph $\mathcal{G} = (\mathcal{V}, \mathcal{E})$, where $\mathcal{V}$ is the set of nodes and $\mathcal{E}$ is the set of hyperedges. The incidence matrix $\H \in \mathbb{R}^{|\mathcal{V}| \times |\mathcal{E}|}$ encodes hypergraph connectivity, with $\H_{ij}=1$ if node $i$ belongs to hyperedge $j$, and $0$ otherwise. Each node is associated with a feature vector, forming a feature matrix $\X \in \mathbb{R}^{|\mathcal{V}| \times d}$ where $d$ is the feature dimension. The class label for each node is encoded in a one-hot matrix $\Y \in \{0,1\}^{|\mathcal{V}| \times c}$, where $c$ is the number of classes.

For node classification, we are given a set $\mathcal{V}_\mathrm{trn} \subset \mathcal{V}$ of training nodes and a set $\mathcal{V}_\mathrm{test} \subset \mathcal{V}$ of test nodes such that $\mathcal{V}_\mathrm{trn} \cap \mathcal{V}_\mathrm{test} = \emptyset$.
We define two diagonal indicator matrices $\Dtrn \in \mathbb{R}^{|\mathcal{V}| \times |\mathcal{V}|}$ and $\Dtest \in \mathbb{R}^{|\mathcal{V}| \times |\mathcal{V}|}$, where $(\Dtrn)_{ii}=1$ iff $i \in \mathcal{V}_\mathrm{trn}$ and $(\Dtest)_{ii}=1$ iff $i \in \mathcal{V}_\mathrm{test}$.
In the $k$-shot setting, we assume that exactly $k$ labeled nodes are available per class, i.e., $\mathrm{tr}(\Dtrn) = kc$.
The goal of $k$-shot node classification is to find a model $f$ that produces prediction $\Yhat \in \mathbb{R}^{|\mathcal{V}| \times c}$ such that:
\begin{equation}
    \Dtest\Yhat \approx \Dtest\Y,
\end{equation}
where $\Yhat = f(\H, \X, \Dtrn\Y)$.
That is, the model is trained using only the labeled nodes in $\Dtrn\Y$, and evaluated on its ability to generalize to test nodes specified by $\Dtest$.
Note that $f$ is a general representation of both training-based and training-free node classifiers.





\subsection{Related work}

\mypara{Hypergraph neural networks (HNNs)}
HNNs have emerged as powerful tools for modeling higher-order relationships, where each hyperedge can connect an arbitrary number of nodes \cite{Feng_You_Zhang_Ji_Gao_2019,NEURIPS2019_1efa39bc}. This expressive capacity enables applications in diverse domains such as recommendation systems, biological networks, and knowledge graphs. Early models such as HGNN \cite{Feng_You_Zhang_Ji_Gao_2019} extend graph convolution to the hypergraph domain via spectral approximation. UniGNN \cite{huang2021unignnunifiedframeworkgraph} provides a unified message-passing framework for both graphs and hypergraphs, while AllSet \cite{chien2022allsetmultisetfunctionframework} introduces a permutation-invariant design that separates set encoding from message propagation. More recently, ED-HNN \cite{wang2023equivarianthypergraphdiffusionneural} explores the connection between the general class of hypergraph diffusion algorithms and the design of HNNs to improve expressiveness in hypergraph settings. Despite their representational power, these models often suffer from overfitting or scalability issues when training labels are scarce.



\mypara{Linear graph neural networks (GNNs)}
GNNs are highly effective for learning node representations from relational data \cite{9046288, kipf2017semisupervisedclassificationgraphconvolutional, NIPS2017_5dd9db5e}. However, their reliance on repeated non-linear transformations across layers leads to high computational costs and limited scalability. To mitigate this, several linear GNN variants have been proposed, which simplify architecture by removing intermediate non-linearity and decoupling transformation from propagation \cite{pmlr-v97-wu19e, gasteiger2022predictpropagategraphneural}. For instance, SGC \cite{pmlr-v97-wu19e} eliminates nonlinearities and applies a fixed propagation matrix multiple times after a single feature transformation. APPNP \cite{gasteiger2022predictpropagategraphneural} adopts a personalized PageRank-based propagation scheme that retains a residual connection to the input features, thereby mitigating over-smoothing without increasing parameter count. $\mathrm{S^2}$GC \cite{zhu2021simple} improved SGC by manually adjusting the strength of self-loops, increasing the number of propagation steps. More recent models such as SlimG \cite{yoo2023} explore the linearized form of GNNs and further improve generalizability and interpretability. These approaches demonstrate that linear architectures can match or surpass nonlinear GNNs in performance, especially with sparse labels.


%% file: sections/030methodology.tex
\section{Proposed method: \method}
\label{sec:method}

We introduce \method, a linear hypergraph neural network (HNN) for fast, scalable, and generalizable node classification in few-shot settings.
Our method builds on a unified linear formulation of existing HNNs (Section \ref{ssec:linearization}).
Leveraging this formulation, we develop a closed-form approximation for the weight matrix (Section \ref{ssec:closed-form}) to eliminate iterative training and propose a redundancy-aware design for the propagation matrix (Section \ref{ssec:propagation}) to mitigate structural redundancy.




\input{sections/031linearization}

\input{sections/032closed-form}
\input{sections/033propagation}
\input{sections/034final-form}

%% file: sections/031linearization.tex
\subsection{General form of a linearized HNN}
\label{ssec:linearization}

Linearization \citep{pmlr-v97-wu19e, yoo2023, wang2021dissecting, li2022g} can effectively simplify the formulation of machine learning models and reduce their computational complexity, resulting in improved robustness, scalability, and interpretability.
We conduct the linearization of five representative HNNs: HGNN \cite{Feng_You_Zhang_Ji_Gao_2019}, HNHN \cite{dong2020hnhn}, UniGCNII \cite{huang2021unignnunifiedframeworkgraph}, AllDeepSets \cite{chien2022allsetmultisetfunctionframework}, ED-HNN \cite{wang2023equivarianthypergraphdiffusionneural}, and introduce their linearized forms in Table \ref{tab01: Linearization}.
We provide the full process of linearization in Appendix~\ref{appendix:linear}.
We remove all non-linear functions, including the activation functions, and integrate multiple weight matrices multiplied together into the single equivalent matrix $\W$ of size $d \times c$.



Then, we generalize the linearized forms of HNNs as follows:
\begin{equation}
    \textstyle \Yhat=\left(\sum_{l=0}^{L}\alpha_l\A_l\right)\X\W,
\label{eq:linear-HNNs}
\end{equation}
where $\A_l$ denotes the $l$-hop adjacency matrix created from $\H$, and $\alpha_l$ is a hyperparameter.
The matrix $\A_l$ determines how the incidence matrix is converted into an adjacency matrix between nodes, while $\alpha_l$ determines how much information is propagated from the $l$-hop neighbors.
It is noteworthy that $\A_l$ does not contain any learnable parameters, since all parameters are already integrated into $\W$. 
For simplicity, we denote $\P = \sum_l \alpha_l\A_l$ and call it a \emph{propagation matrix} in the rest of this paper.

The formulation in Eq. \eqref{eq:linear-HNNs} reveals that the performance of a linear HNN is primarily determined by two factors: (a) the design of the propagation matrix $\P$, which aggregates multi-hop structures $\A_l$ with appropriate coefficients $\alpha_l$, and (b) the optimization of the weight matrix $\W$. We elaborate on the principled design of the propagation matrix in Section \ref{ssec:propagation}. In Section~\ref{ssec:closed-form}, we propose a tractable closed-form solution for $\W$, which eliminates the need for iterative training via backpropagation.

\input{tables/linear}

%% file: tables/linear.tex
\begin{table}[t]
  \caption{Linearized forms of five representative HNNs. All parameters in each layer are integrated into a single weight matrix $\W$ due to the linearization, and the $l$-hop adjacency matrix $\mathbf{A}_l$ is generated from the incidence matrix $\H$ without involving any additional parameters.}
  \label{tab01: Linearization}
  \centering
  \resizebox{\textwidth}{!}{
  \begin{tabular}{l|l|l}
    \toprule
    Method & Linearized form & Adjacency matrix \\
    \midrule
    HGNN & $\Yhat = \A_L \X \W$ & $\A_l = (\Dvh \H \Def \H^\top \Dvh)^l$ \\
    HNHN & $\Yhat = \A_L \X \W$ & $\A_l = (\D_{\mathrm{v},\alpha}^{-1} \H \De^{\alpha} \D_{\mathrm{e},\beta}^{-1} \H^\top \Dv^{\beta})^l$ \\
    UniGCNII & $\Yhat = \left( \sum_{l=0}^{L-1} \alpha(1-\alpha)^l \A_l + (1-\alpha)^L \A_L \right) \X \W$ & $\A_l = (\Dvf \H \tilde{\D}_\mathrm{e}^{-1} \H^\top)^l$ \\
    AllDeepSet & $\Yhat = \A_L \X \W$ & $\A_l = (\Dvf \H \Def \H^\top)^l$ \\
    ED-HNN & $\Yhat = \left( \sum_{l=0}^{L-1} \alpha(1-\alpha)^l \A_l + (1-\alpha)^L \A_L \right) \X \W$ & $\A_l = (\Dvf \H \Def \H^\top)^l$ \\
    \bottomrule
  \end{tabular}
  }
\end{table}

 

    
    
    
    

%% file: sections/032closed-form.tex
\subsection{Tractable closed-form solution for the weight matrix $\W$}
\label{ssec:closed-form}

There are two approaches to obtain the optimal weight matrix for linear HNNs.
The first approach is to iteratively update $\W$ through backpropagation until it reaches a steady state.
Although it is a popular and reasonable approach, we argue that the linear characteristic of the model is not fully exploited in this way.
The second approach is to derive a closed-form solution.
Given training labels $\Dtrn\Y$, we directly compute the optimal $\W^*$ as a function of $\Dtrn\Y$ without having any iterative process.
The limitation is its large computational complexity; it is rarely used in practice due to the cubic time complexity $O(d^3)$ required to compute the pseudoinverse of a matrix.


Therefore, we propose a tractable approximation of the cloesd-form solution to eliminate the need for a pseudoinverse, while maximizing its scalability by removing dependence on iterative learning.
To fully linearize the objective function, we consider the squared error (SSE) loss instead of the typical cross-entropy loss which involves the non-linear softmax function.
The closed-form solution of $\W$, without any approximation or optimization, is given as Lemma \ref{lemma:closed-form-weight}.

\begin{lemma}
\label{lemma:closed-form-weight}
Given a linear HNN $\Yhat = \P\X\W$ and the SSE loss $\mathcal{L_\mathrm{SSE}} = \| \Dtrn\Yhat - \Dtrn \Y\|_\mathrm{F}^2$, the closed-form solution $\W^*$ that minimizes $\mathcal{L_\mathrm{SSE}}$ is given by 
\begin{equation}
    \W^*=((\P\X)^\top \Dtrn (\P\X))^\dagger(\P\X)^\top\Dtrn\Y,
\label{eq:closed-form-weight}
\end{equation}
where $\dagger$ denotes the Moore-Penrose inverse (or the pseudoinverse) of a matrix.
\end{lemma}

\begin{proof}
The full proof is provided in Appendix \ref{appendix:proof-lemma-1}.
\end{proof}

The problem of Eq. \eqref{eq:closed-form-weight} is the pseudoinverse of matrix $(\P\X)^\top \Dtrn (\P\X) \in \mathbb{R}^{d \times d}$, whose computational complexity is $O(d^3)$.
To avoid the pseudoinverse by safely approximating the closed-form solution, we introduce two assumptions on the matrix $\P\X$ which we call the \emph{embedding matrix} of nodes before being mapped to the class space by the matrix $\W$.

\begin{assumption}
    Each row vector of $\P\X$ has a unit norm, i.e., $(\P\X)_{i,:}(\P\X)_{i,:}^\top=1$ for all $i$.
\label{assumption:emb-norm}
\end{assumption}

\begin{assumption}
    For any two row vectors of $\P\X$, their intra-class cosine similarity is approximately $1-\epsilon$, while their inter-class cosine similarity is approximately $\epsilon$. That is, $(\P\X)_{i,:}(\P\X)_{j,:}^\top \approx 1-\epsilon$ for $i, j$ in the same class, while \smash{$(\P\X)_{i,:}(\P\X)_{j,:}^\top \approx \epsilon$} for $i, j$ in different classes, with small $\epsilon$.
\label{assumption:emb-good}
\end{assumption}

Assumption \ref{assumption:emb-norm} can be easily achieved by normalizing the embedding matrix $\P\X$.
Assumption \ref{assumption:emb-good} is harder to satisfy, but is reasonable to assume in many cases where the node feature matrix $\X$ provides sufficient information for classification when combined with the structural information encoded in $\P$, especially when we loosen the error bound $\epsilon$.
Then, we propose Theorem \ref{theorem:approximation} for approximating the closed-form solution of $\W$ with the introduced assumptions.

\begin{theorem}
\label{theorem:approximation}
Under Assumption \ref{assumption:emb-norm} and \ref{assumption:emb-good}, the following holds:
\begin{equation*}
\W^*=((\P\X)^\top \Dtrn (\P\X))^\dagger(\P\X)^\top\Dtrn\Y \approx \frac{1}{\epsilon}(\P\X)^\top\Dtrn\Y.
\end{equation*}
\end{theorem}

\begin{proof}
The full proof is provided in Appendix \ref{appendix:proof-theorem-1}.
\end{proof}

By Theorem \ref{theorem:approximation}, the burdensome pseudoinverse in the optimal weight matrix can be approximated as $(1/\epsilon)\I$, which is a constant matrix.
Since the constant $1/\epsilon$ is multiplied to all dimensions, it does not change the result of prediction; we can safely remove it from our model.


To satisfy Assumption \ref{assumption:emb-norm}, we apply the row-wise L2 normalization for the embedding matrix $\P\X$ as $g_\mathrm{row}(\P\X)$ where the function $g_\mathrm{row}$ denotes the row-wise normalization operator.
Consequently, we normalize the weight matrix $\W^*$ as well, since it represents the \emph{embedding matrix of labels;} each column of it can be understood as the $d$-dimensional embedding of each label. Without normalizing $\W^*$, the class with a large-norm embedding gets an advantage in the prediction stage.
Therefore, we normalize both node and class embeddings, and our final method is expressed as follows:
\begin{equation}
    \Yhat = g_\mathrm{row}(\P\X)g_\mathrm{col}(g_\mathrm{row}(\P\X)^\top\Dtrn\Y).
\label{eq:normalization}
\end{equation}

%% file: sections/033propagation.tex
\subsection{Redundancy-aware propagation for eliminating self-information}
\label{ssec:propagation}

A key design objective in linear HNNs is to flexibly control the influence of each $l$-hop neighborhood on node representations. This is achieved through a propagation matrix of the form
\begin{equation}
\P = \sum_{l=0}^L \alpha_l \A_l,
\label{eq:propagation-matrix}
\end{equation}
which denotes a weighted combination of $l$-hop adjacency matrix $\A_l$ multiplied with the coefficient $\alpha_l$. Since $\P\X$ is row-normalized as shown in Eq. \eqref{eq:normalization}, without loss of generality, we can constrain the coefficients to lie on the probability simplex: \smash{$\sum_{l=0}^L \alpha_l=1$,} where $\alpha_l\ge0$ for all $l$.
By carefully controlling the coefficients for each dataset, determining how much information to take from each $l$-hop neighborhood, a linear HNN can exhibit superior performance. 
For example, $\alpha_1$ can be large for homophilic graphs, but small for heterophilic graphs following a typical assumption.

However, the formulation in Eq. \eqref{eq:propagation-matrix} inevitably contains \textit{redundant self-information}, which hinders the precise adjustment of neighborhood influence. We formalize this phenomenon as follows:

\begin{definition}
    Given $\A_l$ with $l>0$, we define its \textit{residual self-information} as the diagonal matrix
    \begin{equation}
        \RSI(\A_l) := 
        \operatorname{diag}(\A_l) \in \mathbb{R}^{n \times n},
    \end{equation}
    which quantifies the amount of information returning to each self-node.
\end{definition}

The self-information of $\mathbf{A}_l$ can arise not only from the self-loops, but also from cycles or return paths in hypergraph walks if $l \ge 2$.
Such self-information is not a notable problem, but even promoted in nonlinear HNNs which aim to avoid losing the information of initial node features when deep layers are adopted.
However, in linear HNNs, self-information repeatedly appearing in different-hop adjacency matrices is redundant and prevents one from fully optimizing the propagation function for each dataset with a unique characteristic.
Another problem is that the self-information in $\mathbf{A}_l$, if it is used for creating $\mathbf{A}_{l'}$ with $l' > l$, results in boosting the effect of local neighborhoods in $\mathbf{A}_{l'}$, making the self-information even stronger in later adjacency matrices.




To overcome these limitations, we propose a \textit{redundancy-aware propagation} that explicitly eliminates redundant self-information from the $l$-hop adjacency matrix $\A_l$. Specifically, we define the $l$-hop adjacency matrix $\A_l^*$ without redundant self-information as follows:
\begin{equation}
    \A_l^* := \hat{\A}_l - \RSI(\hat{\A}_l),
\label{eq:our-propagation}
\end{equation}
where $\hat{\A}_l := \A^*_1 \D_{l-1}^* \A_{l-1}^*$ is the normalized adjacency matrix, $\D_{l-1}^*$ is the degree matrix designed specifically for the normalization method used in creating $\A_l$, and $\A_0^* := \I$.
Then, we replace all $\A_l$ with $\A^*_l$ in Eq. \eqref{eq:propagation-matrix} and use it as our propagation matrix $\P^* = \sum_l \alpha_l \A_l^*$.
This ensures that the self-information in $\P^*$ is exactly determined by $\alpha_0$, and is not confounded by higher-hop structures. 
As a result, it allows a precise control over self- versus neighbor- information, which is critical in few-shot scenarios where overfitting to redundant self-signals can hinder generalization.

\mypara{Normalization of adjacency matrices}
All linearized HNNs in Table \ref{tab01: Linearization} normalize the adjacency matrix $\mathbf{A}_l$ either by symmetric or row normalization to make it numerically stable in deep layers.
However, if we apply the same original normalization and remove self-information from normalized $\mathbf{A}_l$, it makes over-normalization and gradually diminishes feature magnitudes across layers.
Thus, we re-normalize $\A_l$ and derive the following matrices for $l=1, 2$:
\begin{align}
    &\hat{\A}_1 = \Dvh \H (\De-\I)^{-1} \H^\top \Dvh \label{eq:normalized-adj-1} \\
    &\hat{\A}_2 = \A_1^*
    \left(\Dv^\frac{1}{2}(\Dv-\I)^{-1}\Dv^\frac{1}{2}\right)\A_1^*. \label{eq:normalized-adj-2}
\end{align}

\mypara{Avoiding dense adjacency matrices}
The main challenge for Eq. \eqref{eq:our-propagation} is deriving the dense matrix $\hat{\A}_l$ of size $|\mathcal{V}| \times |\mathcal{V}|$ for the computation of its self-information.
We provide explicit expressions of the RSI terms for $l=1$ and $l=2$, since we set $L$ to 2 in our experiments.
Our derivation is based on symmetric normalization, but the framework is readily extensible to row normalization. Detailed derivations for the row-normalized variant are included in Appendix~\ref{appendix:row-norm}.


\begin{lemma}
Given $\hat{\A}_1$ in Eq. \eqref{eq:normalized-adj-1}, $\RSI(\hat{\A}_1)$ is given by
\begin{equation}
\label{eq:rsi1}
\textstyle (\RSI(\hat{\A}_1))_{ii} = d_{v_i}^{-1} \left(\sum_{e_j \in \mathcal{N}(v_i)} (d_{e_j} - 1)^{-1} \right),
\end{equation}
where $d_x$ denotes the degree of node $x$ or the number of nodes in hyperedge $x$, based on the type of $x$, and $\mathcal{N}(v_i)$ denotes the set of hyperedges incident to node $v_i$.
\label{lemma:rsa1}
\end{lemma}

\begin{proof}
    The full proof is provided in Appendix \ref{appendix:proof-lemma-2}.
\end{proof}



\begin{lemma}
Given $\mathbf{A}_1^* = \hat{\mathbf{A}}_1 - \RSI(\hat{\mathbf{A}}_1)$ and $\hat{\A}_2$ in Eq. \eqref{eq:normalized-adj-2}, $\RSI(\hat{\A}_2)$ is given by
\begin{equation}
\label{eq:rsi2}
\textstyle
(\RSI(\hat{\A}_2))_{ii} = d_{v_i}^{-1} \left( \sum_{e_j \in \mathcal{N}(v_i)} (d_{e_j} - 1)^{-2} \left( \sum_{v_k \in \mathcal{N}(e_j) \setminus \{v_i\}} (d_{v_k} - 1)^{-1} \right) \right),
\end{equation}
where $d_x$ denotes the degree of node $x$ or the number of nodes in hyperedge $x$, based on the type of $x$, and $\mathcal{N}(v_i)$ denotes the set of hyperedges incident to node $v_i$.
\end{lemma}

\begin{proof}
    The proof follows by applying the same reasoning as in Lemma \ref{lemma:rsa1}.
\end{proof}


From the 3-hop neighborhood onward, self-information can return through cycles rather than simple backtracking paths, complicating both its identification and principled removal during propagation. Additionally, deeper propagation exacerbates the computational burden and enlarges the hyperparameter space. For these reasons, we restrict our model to 2-hop propagation, which strikes a balance between expressive power and efficiency. This design choice is supported by two key observations:
(a) Increasing the number of hops introduces a linearly growing number of propagation coefficients $\alpha_l$, resulting in the quadratic expansion of the hyperparameter space.
(b) Empirically, most HNNs attain competitive performance with only 1–2 propagation layers.

%% file: sections/034final-form.tex
\subsection{Summary of \method and its computational complexity}

Putting it all together, the \method classifier $f^*$ can be summarized as follow:
\begin{equation}
    f^*(\H, \X, \Dtrn\Y) = g_\mathrm{row}(\P^*\X)g_\mathrm{col}(g_\mathrm{row}(\P^*\X)^\top\Dtrn\Y),
\end{equation}
where $\P^* =\alpha_0\A_0^*+\alpha_1\A_1^*+\alpha_2\A_2^*$ with three hyperparameters $\alpha_0$, $\alpha_1$, and $\alpha_2$, and $\A_l^*$ denotes the refined $l$-hop adjacency matrix which eliminates the redundant self-information.

The function $f^*$ is a closed-form prediction formula that directly produces label predictions for all nodes given a small set of labeled nodes.
Crucially, it eliminates the need for iterative training: the labels of the training nodes are injected as explicit inputs rather than being implicitly encoded via gradient-based optimization. 
This design enables extremely fast and stable inference, particularly under low-resource scenarios such as few-shot settings, where traditional training-based models often suffer from instability due to limited supervision. 
More importantly, by removing self-information explicitly, \method avoids redundancy and enables fully optimized combinations of multi-hop structure within the probability simplex. This capability is fundamentally absent in conventional propagation schemes, where self-information is entangled with higher-hop structures.

The computational complexity of $f^*$ depends on the two main stages.
The first is the dense matrix multiplication between $g_\mathrm{row}(\P^*\X)$ and  $g_\mathrm{col}(g_\mathrm{row}(\P^*\X)^\top\Dtrn\Y)$, which is $O(|\mathcal{V}|dc)$, linear in the number of nodes, feature dimension, and number of classes.
The second is the construction of $\P^*$, where we remove self-information from each $\mathbf{A}_l$ for all $l\ge1$.
This has the same time complexity as standard message passing schemes, $O(\text{nnz}(\H)\cdot dL)$, where $\text{nnz}(\H)$ denotes the number of non-zero entries in the hypergraph structure, and the number $L$ of layers is set to $2$ in our experiments.



%% file: sections/040experiment.tex
\section{Experiments}
\label{sec:exp}



  

We conduct comprehensive experiments on 11 real-world hypergraphs to verify the effectiveness of \method. We show that \method consistently outperforms existing HNNs in few-shot node classification tasks while exhibiting remarkable scalability.
We then present a case study highlighting the interpretability of \method, which comes from its linear decision mechanism, on a real hypergraph.






\subsection{Experimental setup}

\mypara{Datasets.}
We evaluate \method on a total of 11 real-world hypergraph graphs. To assess predictive performance and computational efficiency, we use 10 standard benchmarks: Cora, Citeseer, Pubmed, Cora\_CA, 20News100, ModelNet40, Congress, Walmart, Senate, and House, following prior work \cite{wang2023equivarianthypergraphdiffusionneural, Lee_Shin_2023}. 
For interpretability analysis, we use Zoo \cite{Lee_Shin_2023}, a small dataset whose feature attributes are semantically interpretable.
Detailed dataset statistics are provided in Table~\ref{tab02: data statistic}.

\input{tables/dataset_statistic}

\mypara{Baselines and hyperparameters.}
We compare \method with 8 representative hypergraph neural networks (HNN) models: HGNN \cite{Feng_You_Zhang_Ji_Gao_2019}, HNHN \cite{dong2020hnhn}, HCHA \cite{bai2021hypergraph}, UniGCN, UniGCNII \cite{huang2021unignnunifiedframeworkgraph}, AllDeepSets, AllSetTransformer \cite{chien2022allsetmultisetfunctionframework}, and ED-HNN \cite{wang2023equivarianthypergraphdiffusionneural}.
All baselines are implemented based on the official codebase of ED-HNN, which provides a unified framework for fair comparison.
All baseline are trained using the Adam optimizer with no weight decay, and we conduct a grid search over 72 hyperparameter configurations:
$\mathrm{lr} \in \{10^{-2}, 10^{-3}, 10^{-4}\}, \
\mathrm{epochs} \in \{50, 100, 150, 200\},\
\mathrm{num\_layers} \in \{1, 2\},\
\mathrm{hidden\_dim} \in \{64, 128, 256\}.$
In contrast, \method requires no training hyperparameters.
Instead, we search over 55 combinations of propagation coefficients \((\alpha_0, \alpha_1, \alpha_2)\) uniformly sampled from the 2-simplex, yielding a comparable hyperparameter space size to that of baselines. For each dataset split, we report the test accuracy corresponding to the best validation performance. All our experiments are conducted with NVIDIA RTX A6000 and AMD EPYC 9354.



\mypara{Evaluation.}
We evaluate the accuracy of all models on 10 random data splits per dataset. For each split, we allocate 5 labeled nodes per class for training, and another 5 nodes per class for validation \cite{kim2024hypeboy, wu2022hypergraph}, making 5-shot node classification. The remaining nodes are used for testing. We report the average classification accuracy and standard deviation across the ten splits \cite{kim2023datasets, li2024hypergraph, lee2024villain}.


\input{tables/accuracy_5nodes}
\subsection{Classification accuracy}

Table \ref{tab05: small dataset} compares the accuracy of \method and the baseline HNNs on 10 hypergraphs.
\method demonstrates competitive or superior performance across all datasets, showing the highest average rank.
Despite its simple linear architecture, \method achieves high accuracy even on complex hypergraph structures, being competitive with complicated nonlinear methods. This validates the effectiveness of \method's architectural design in capturing high-order relationships, and highlights its strong generalization ability in few-shot node classification scenarios, where model robustness is crucial.

The results also highlight intriguing trends among baseline models.
In particular, early models such as HGNN and UniGCN remain competitive, particularly on datasets such as 20News and Congress. Their relatively simple architectures may contribute to stronger generalization in few-shot settings, where overfitting is a common challenge.
In contrast, more complex methods such as HNHN and AllDeepSets tend to underperform, likely due to higher complexity and reduced robustness under limited supervision. These observations further underscore the strength of the simple yet effective design of \method, producing consistently high performance in diverse hypergraph structures.


\subsection{Running time}

Table \ref{tab05: time ratio small dataset} shows the running time of \method and the baseline models, including both training and inference time.
All existing HNNs exhibit significantly higher computational costs over \method.
The speedup stands out for large complicated HNNs, such as AllDeepSets and ED-HNN, as \method is over 1700$\times$ faster than AllDeepSets in the House dataset.
Even for relatively simple models such as HGNN, HCHA, UniGCN, and UniGCNII, \method shows a consistent improvement ranging from 2.5$\times$ even to 300$\times$.
In summary, \method demonstrates overwhelming speed superiority across all datasets while maintaining competitive accuracy, establishing itself as a highly efficient solution.

The efficiency of \method comes from its lightweight architecture, where the computational cost is linear in the number of nodes, feature dimension and number of classes. This design leads to substantial runtime advantages on datasets with compact input dimensions. In particular, Congress, Senate, and House have feature dimensions that are at least \(5\times\) and up to \(37\times\) smaller, node counts up to \(9\times\) fewer, and class counts between \(2\times\) and \(20\times\) fewer than other datasets. These characteristics make them ideal for highlighting the scalability of \method, which achieves remarkable speedups of up to \(292\times\) on House and \(696\times\) on Senate while maintaining competitive accuracy.

\input{tables/speed}


\subsection{Interpretability}
\label{ssec:interpretability}

\method is inherently interpretable, thanks to its linear decision process that directly maps the feature space to the prediction.
There are two ways to interpret the learned knowledge of \method.
First, the $(i, k)$-th element of the weight matrix $\mathbf{W}^*$ can be understood as the importance of the $i$-th feature for predicting the $k$-th class.
Second, each column of $\mathbf{W}^*$ can be understood as the embedding of class $k$ lying on the feature space.
In this way, the relationships between class embeddings and node features provide a deeper insight on the nature of the given dataset, along with the graphical structure.

To verify the interpretability of \method, we conduct a case study on the Zoo dataset, whose node features have clear semantic meanings: each feature attribute represents a characteristic of an animal, e.g., hair or milk, while the target class is its species.
The nodes represent animals, and the hyperedges group together animals that share a common feature, e.g., all animals having the same \emph{hair}.

Table~\ref{tab07: interpretability} visualizes the weight matrix $\W^*$, where each value is color-coded: darker red indicates a higher relative value compared to other classes for that feature.
Since the initial input features are all nonnegative, the resulting weight elements also remain nonnegative.
For instance, the Mammal class shows significantly higher values in Hair, Milk, and Capsize, suggesting that these features play a key role in distinguishing mammals from other animal groups. Similarly, the Bird class exhibits prominent values in Feathers, Eggs, Airborne, etc., reflecting biologically distinctive traits of birds.
These results demonstrate that \method not only achieves high predictive performance but also yields representations that align with domain knowledge in a transparent and interpretable manner.

\input{tables/table_interpretability}
\subsection{Ablation study}

In Table~\ref{tab07: ablation}, we conduct an ablation study to assess the individual and combined impact of \emph{tractable closed-form solution} (TCS) and \emph{redundancy-aware propagation} (RAP), two core modules of \method, in the same setting as in Table \ref{tab05: small dataset}.
The results show that the removal of both components leads to notable performance drops, while isolating TCS or RAP yields moderate gains, with TCS generally offering more stable improvements.
Our full model \method consistently achieves the best results, outperforming all ablated variants with the highest average rank.
These findings highlight the complementary roles of TCS and RAP—TCS provides stability and tractability through a closed-form solution, while RAP enhances representation by mitigating redundancy.

We observe that removing RAP or TCS can improve performance on certain datasets (Congress and Senate for RAP, 20News and MN40 for TCS).
RAP tends to degrade performance on high-density datasets such as Congress and Senate, likely because the degree normalization reduces each node’s relative contribution, limiting the effect of self-information. For TCS, which assumes informative node embeddings, the impact of removal is more significant on 20News than MN40. The extremely low density of 20News may hinder \method, which relies solely on propagation to refine embeddings, from producing sufficiently informative representations. In contrast, MN40’s large number of classes increases the relative Frobenius norm error of TCS, suggesting that the approximation becomes less accurate as the number of classes grows, partially explaining the observed results.

\input{tables/ablation}




%% file: tables/dataset_statistic.tex
\begin{table}[t]
  \caption{Statistics of datasets. The first ten datasets are used as the main benchmark for evaluating accuracy, while the Zoo dataset is used for interpretability analysis.}
  \label{tab02: data statistic}
  \centering
  \resizebox{\textwidth}{!}{
  \begin{tabular}{l|rrrrrrrrrr|r}
    \toprule
       &Cora & Citeseer& Pubmed & Cora-CA & 20News & MN40 &   Congress & Walmart & Senate & House & Zoo \\
    \midrule
    \midrule
    \# nodes $|\mathcal{V}|$&  2,708 & 3,312 & 19,717 & 2,708 & 16,242 & 12,311& 1,718 & 88,860 & 282 &1,290 & 101 \\
    \# edges $|\mathcal{E}|$ & 1,579 & 1,079 & 7,963 & 1,072& 100 & 12,311 & 83,105 & 69,906 &315 &340& 43 \\
    \# classes $c$  & 7 & 6 & 3 & 7& 4& 40 &2 &11 & 2 & 2& 7  \\
    
    \# features $d$  & 1,433 & 3,703 & 500 & 1,433 & 100  &100 & 100 & 100 & 100 &100 &16 \\
    density ($\frac{|\mathcal{E}|}{|\mathcal{V}|}$) & 0.5835 & 0.3258 & 0.4041 & 0.3959 & 0.0062 & 1.0000 & 48.3946 & 0.7868 & 1.1160 & 0.2636 & 0.4257\\
    \bottomrule
  \end{tabular}
  }
\end{table}

%% file: tables/accuracy_5nodes.tex
\begin{table}[t]
  \caption{Classification accuracy (\%) for 5-shot node classification on real-world hypergraphs. We report the mean and standard deviation over 10 runs. Boldfaced letters indicate the best accuracy, and underlined letters indicate the second.
  \method achieves the highest average rank.}
  \label{tab05: small dataset}
  \centering
  \setlength{\defaultaddspace}{0.5ex}
  \resizebox{\textwidth}{!}{
  \begin{tabular}{l|cccccccccc|c}
    \toprule
    
     \textbf{Methods} & \textbf{Cora} & \textbf{Citeseer} & \textbf{Pubmed} & \textbf{Cora\_CA} & \textbf{20News} &
     \textbf{MN40}&\textbf{Congress}&\textbf{Wallmart} &\textbf{Senate}&\textbf{House}
     &\begin{tabular}[c]{@{}c@{}}\textbf{Avg.}\\\textbf{Rank}\end{tabular}  \\
    \midrule
    \midrule
    HGNN & $44.4_{\pm 8.9}$ & $40.1_{\pm 6.5}$ & $52.5_{\pm 9.1}$ & $54.3_{\pm 3.6}$ & $\mathbf{73.1_{\pm 2.3}}$ & $94.7_{\pm 0.3}$ & $86.7_{\pm 1.1}$ & $39.6_{\pm 2.4}$ & $56.8_{\pm 5.0}$ & $63.4_{\pm 4.3}$ & $5.9$ \\
    \addlinespace
    HNHN & $36.7_{\pm 5.8}$ & $36.0_{\pm 3.7}$ & $51.8_{\pm 3.7}$ & $39.2_{\pm 5.2}$ & $41.2_{\pm 5.7}$ & $90.8_{\pm 1.4}$ & $51.1_{\pm 2.7}$ & $15.9_{\pm 3.0}$ & $69.7_{\pm 11.6}$ & $67.4_{\pm 8.3}$ & $7.9$ \\
    \addlinespace
    HCHA & $44.4_{\pm 8.7}$ & $41.2_{\pm 6.5}$ & $52.9_{\pm 10.4}$ & $54.5_{\pm 4.2}$ & $\underline{72.9_{\pm 2.5}}$ & $94.7_{\pm 0.2}$ & $86.6_{\pm 1.3}$ & $39.3_{\pm 2.5}$ & $53.0_{\pm 5.0}$ & $63.5_{\pm 4.6}$ & $5.9$ \\
    \addlinespace
    UniGCN & $48.5_{\pm 8.3}$ & $41.6_{\pm 3.7}$ & $54.2_{\pm 10.3}$ & $55.3_{\pm 4.3}$ & $70.4_{\pm 2.8}$ & $95.9_{\pm 0.3}$ & $\mathbf{91.6_{\pm 0.4}}$ & $40.1_{\pm 2.8}$ & $61.4_{\pm 4.4}$ & $67.9_{\pm 5.1}$ & $3.9$ \\
    \addlinespace
    UniGCNII & $43.3_{\pm 9.9}$ & $38.9_{\pm 6.7}$ & $54.5_{\pm 8.4}$ & $52.0_{\pm 4.5}$ & $66.5_{\pm 4.6}$ & $\underline{96.4_{\pm 0.4}}$ & $83.5_{\pm 6.4}$ & $23.5_{\pm 2.4}$ & $\mathbf{70.4_{\pm 8.5}}$ & $\underline{70.7_{\pm 7.4}}$ & $5.5$ \\
    \addlinespace
    AllDeepSets & $48.6_{\pm 4.7}$ & $42.6_{\pm 4.4}$ & $53.2_{\pm 5.8}$ & $55.3_{\pm 5.1}$ & $51.4_{\pm 4.4}$ & $94.7_{\pm 0.3}$ & $69.5_{\pm 4.7}$ & $24.5_{\pm 3.7}$ & $65.3_{\pm 10.3}$ & $63.4_{\pm 8.3}$ & $5.7$ \\
    \addlinespace
    AllSetTransformer & $\underline{50.5_{\pm 4.4}}$ & $\underline{44.8_{\pm 2.7}}$ & $\underline{60.4_{\pm 4.5}}$ & $\underline{59.6_{\pm 3.4}}$ & $70.3_{\pm 1.5}$ & $95.5_{\pm 0.2}$ & $88.2_{\pm 1.1}$ & $38.3_{\pm 6.4}$ & $63.1_{\pm 10.7}$ & $66.3_{\pm 8.3}$ & $\underline{3.6}$ \\
    \addlinespace
    ED-HNN & $48.4_{\pm 6.4}$ & $44.5_{\pm 3.5}$ & $56.5_{\pm 6.6}$ & $58.8_{\pm 3.8}$ & $67.7_{\pm 3.7}$ & $96.0_{\pm 0.2}$ & $\underline{89.1_{\pm 4.0}}$ & $\underline{42.9_{\pm 5.7}}$ & $63.1_{\pm 9.1}$ & $62.8_{\pm 10.4}$ & $4.1$ \\
    \midrule
    \midrule 
    \textbf{\method (proposed)}& $\mathbf{51.9_{\pm 10.1}}$ & $\mathbf{49.1_{\pm 4.8}}$ & $\mathbf{62.6_{\pm 3.9}}$ & $\mathbf{60.0_{\pm 6.2}}$ & $68.6_{\pm 4.8}$ & $\mathbf{97.6_{\pm 0.3}}$ & $87.0_{\pm 4.8}$ & $\mathbf{43.9_{\pm 3.1}}$ & $\mathbf{70.4_{\pm 10.0}}$ & $\mathbf{73.2_{\pm 6.3}}$ & $\textbf{1.7}$\\  
    \bottomrule
  \end{tabular}
  }
\end{table}

%% file: tables/speed.tex
\begin{table}[t]
  \caption{The running time of \method and the baseline models, including both training and inference. Each time is represented as a ratio over the running time of \method. Therefore, the lower is the better.
  \method consistently shows the fastest runtime, up to $696 \times$ faster than the best competitor.}
  \label{tab05: time ratio small dataset}
  \centering
  \setlength{\defaultaddspace}{0.5ex}
  \resizebox{\textwidth}{!}{
  \begin{tabular}{l|rrrrrrrrrr}
    \toprule
    
     \textbf{Methods} & \textbf{Cora} & \textbf{Citeseer} & \textbf{Pubmed} & \textbf{Cora\_CA} & \textbf{20News} &
     \textbf{MN40}&\textbf{Congress}&\textbf{Wallmart} &\textbf{Senate}&\textbf{House}
        \\
    \midrule
    \midrule
    HGNN & $8.65$ & $3.81$ & $3.00$ & $8.95$ & $\underline{10.54}$ & $20.05$ & $53.10$ & $30.16$ & $777.14$ & $388.05$\\
    \addlinespace
    HNHN & $\underline{7.61}$ & $\underline{2.71}$ & $3.56$ & $\underline{7.23}$ & $15.51$ & $\underline{13.78}$ & $\underline{42.11}$ & $21.47$ & $696.85$ & $345.93$\\
    \addlinespace
    HCHA & $12.37$ & $4.74$ & $5.51$ & $10.19$ & $12.21$ & $20.63$ & $71.62$ & $\underline{12.73}$ & $1008.85$ & $699.41$\\
    \addlinespace
    UniGCN & $17.12$ & $2.63$ & $8.29$ & $8.97$ & $28.79$ & $21.86$ & $91.16$ & $32.90$ & $716.43$ & $\underline{292.95}$\\
    \addlinespace
    UniGCNII & $15.77$ & $5.18$ & $\underline{2.56}$ & $17.25$ & $16.98$ & $19.41$ & $80.17$ & $36.37$ & $\underline{696.63}$ & $369.10$\\
    \addlinespace
    AllDeepSets & $58.11$ & $24.92$ & $11.01$ & $57.56$ & $35.52$ & $47.51$ & $273.58$ & $90.16$ & $4048.55$ & $1748.67$\\
    \addlinespace
    AllSetTransformer & $9.76$ & $4.30$ & $10.98$ & $12.37$ & $60.21$ & $24.94$ & $65.59$ & $76.89$ & $997.43$ & $524.15$\\
    \addlinespace
    ED-HNN & $16.04$ & $6.06$ & $5.46$ & $23.55$ & $46.71$ & $28.70$ & $426.91$ & $46.31$ & $714.73$ & $379.41$\\
    \midrule
    \midrule
    \method \textbf{(proposed)}& $\mathbf{1.00}$ & $\mathbf{1.00}$ & $\mathbf{1.00}$ & $\mathbf{1.00}$ & $\mathbf{1.00}$ & $\mathbf{1.00}$ & $\mathbf{1.00}$ & $\mathbf{1.00}$ & $\mathbf{1.00}$ & $\mathbf{1.00}$\\  
    \bottomrule
  \end{tabular}
  }
\end{table}

%% file: tables/table_interpretability.tex
\begin{table}[t]
  \caption{Relative feature importance values learned by \method on the Zoo dataset under a 3-shot setting. Darker red cells indicate higher values relative to other classes, with cells having darkness of 80\% or higher highlighted by black boxes.
  Refer to Section \ref{ssec:interpretability} for detailed information.}  
  \label{tab07: interpretability}
  \centering
  \setlength{\fboxrule}{1pt}
  \setlength{\defaultaddspace}{0.3ex}
  \resizebox{0.8\textwidth}{!}{
  \begin{tabular}{cccccccc}
    \toprule
    
      & \textbf{Mammal} & \textbf{Bird} & \textbf{Reptile} & \textbf{Fish} & \textbf{Amphibian} &
     \textbf{Bug}&\textbf{Invertebrate}\\
    \midrule
    Hair & \cellcolor{red!90}\fbox{$0.1735$} & \cellcolor{red!40}$0.0719$ & \cellcolor{red!40}$0.0778$ & \cellcolor{red!50}$0.0819$ & \cellcolor{red!40}$0.0661$ & \cellcolor{red!50}$0.0800$ & \cellcolor{red!30}$0.0465$ \\
    \addlinespace
    Feathers & \cellcolor{red!20}$0.0251$ & \cellcolor{red!90}\fbox{$0.1726$} & \cellcolor{red!40}$0.0383$ & \cellcolor{red!30}$0.0414$ & \cellcolor{red!20}$0.0311$ & \cellcolor{red!20}$0.0305$ & \cellcolor{red!20}$0.0276$ \\
    \addlinespace
    Eggs & \cellcolor{red!10}$0.1072$ & \cellcolor{red!80}\fbox{$0.2440$} & \cellcolor{red!70}$0.2255$ & \cellcolor{red!100}\fbox{$0.2931$} & \cellcolor{red!50}$0.1979$ & \cellcolor{red!40}$0.1792$ & \cellcolor{red!50}$0.1850$ \\
    \addlinespace
    Milk & \cellcolor{red!90}\fbox{$0.1720$} & \cellcolor{red!40}$0.0664$ & \cellcolor{red!40}$0.0740$ & \cellcolor{red!50}$0.0826$ & \cellcolor{red!40}$0.0626$ & \cellcolor{red!30}$0.0432$ & \cellcolor{red!30}$0.0409$ \\
    \addlinespace
    Airborne & \cellcolor{red!20}$0.0310$ & \cellcolor{red!80}\fbox{$0.1417$} & \cellcolor{red!30}$0.0450$ & \cellcolor{red!30}$0.0456$ & \cellcolor{red!20}$0.0376$ & \cellcolor{red!50}$0.0979$ & \cellcolor{red!20}$0.0361$ \\
    \addlinespace
    Aquatic & \cellcolor{red!20}$0.0807$ & \cellcolor{red!10}$0.0688$ & \cellcolor{red!20}$0.0727$ & \cellcolor{red!100}\fbox{$0.2522$} & \cellcolor{red!60}$0.1606$ & \cellcolor{red!10}$0.0492$ & \cellcolor{red!50}$0.1437$ \\
    \addlinespace
    Predetor & \cellcolor{red!40}$0.1796$ & \cellcolor{red!10}$0.1054$ & \cellcolor{red!50}$0.1878$ & \cellcolor{red!100}\fbox{$0.2841$} & \cellcolor{red!40}$0.1601$ & \cellcolor{red!10}$0.1006$ & \cellcolor{red!40}$0.1688$ \\
    \addlinespace
    Toothed & \cellcolor{red!60}$0.1707$ & \cellcolor{red!30}$0.1001$ & \cellcolor{red!70}$0.1946$ & \cellcolor{red!100}\fbox{$0.2921$} & \cellcolor{red!70}$0.1959$ & \cellcolor{red!10}$0.0682$ & \cellcolor{red!10}$0.0745$ \\
    \addlinespace
    Backbone & \cellcolor{red!70}$0.2280$ & \cellcolor{red!90}\fbox{$0.2768$} & \cellcolor{red!90}\fbox{$0.2687$} & \cellcolor{red!100}\fbox{$0.3380$} & \cellcolor{red!70}$0.2304$ & \cellcolor{red!10}$0.1012$ & \cellcolor{red!10}$0.1049$ \\
    \addlinespace
    Breathes & \cellcolor{red!70}$0.2239$ & \cellcolor{red!90}\fbox{$0.2769$} & \cellcolor{red!80}\fbox{$0.2597$} & \cellcolor{red!40}$0.1652$ & \cellcolor{red!70}$0.2235$ & \cellcolor{red!50}$0.1939$ & \cellcolor{red!10}$0.1064$ \\
    \addlinespace
    Venomous & \cellcolor{red!10}$0.0109$ & \cellcolor{red!10}$0.0141$ & \cellcolor{red!40}$0.0628$ & \cellcolor{red!20}$0.0217$ & \cellcolor{red!30}$0.0465$ & \cellcolor{red!30}$0.0415$ & \cellcolor{red!20}$0.171$ \\
    \addlinespace
    Fins & \cellcolor{red!20}$0.0246$ & \cellcolor{red!20}$0.0272$ & \cellcolor{red!20}$0.0337$ & \cellcolor{red!100}\fbox{$0.2027$} & \cellcolor{red!20}$0.0301$ & \cellcolor{red!10}$0.0183$ & \cellcolor{red!20}$0.0273$ \\
    \addlinespace
    Legs & \cellcolor{red!60}$0.8380$ & \cellcolor{red!50}$0.7794$ & \cellcolor{red!50}$0.7956$ & \cellcolor{red!10}$0.5875$ & \cellcolor{red!80}\fbox{$0.8520$} & \cellcolor{red!90}\fbox{$0.9333$} & \cellcolor{red!90}\fbox{$0.9310$} \\
    \addlinespace
    Tail & \cellcolor{red!50}$0.1836$ & \cellcolor{red!90}\fbox{$0.2642$} & \cellcolor{red!80}\fbox{$0.2542$} & \cellcolor{red!100}\fbox{$0.3204$} & \cellcolor{red!30}$0.1522$ & \cellcolor{red!10}$0.0913$ & \cellcolor{red!10}$0.0963$ \\
    \addlinespace
    Domestic & \cellcolor{red!20}$0.0225$ & \cellcolor{red!20}$0.0266$ & \cellcolor{red!20}$0.0234$ & \cellcolor{red!20}$0.0253$ & \cellcolor{red!10}$0.0197$ & \cellcolor{red!10}$0.0179$ & \cellcolor{red!10}$0.0146$ \\
    \addlinespace
    Capsize & \cellcolor{red!80}\fbox{$0.1416$} & \cellcolor{red!80}\fbox{$0.1515$} & \cellcolor{red!60}$0.1134$ & \cellcolor{red!100}\fbox{$0.1887$} & \cellcolor{red!40}$0.0687$ & \cellcolor{red!30}$0.0487$ & \cellcolor{red!40}$0.0744$ \\
    
    \bottomrule
  \end{tabular}
  }
\end{table}

%% file: tables/ablation.tex
\begin{table}[t]
  \caption{Ablation study of \method with four baselines: three variants with a selective removal of the \method components,
  and the linearized HGNN that lacks multi-hop message combination. Our full model  consistently outperforms all ablations, demonstrating the effectiveness of both components.}

  \label{tab07: ablation}
  \centering
  \setlength{\defaultaddspace}{0.5ex}
  \resizebox{\textwidth}{!}{
  \begin{tabular}{l|cccccccccc|c}
    \toprule
     \textbf{Methods} & \textbf{Cora} & \textbf{Citeseer} & \textbf{Pubmed} & \textbf{Cora\_CA} & \textbf{20News} &
     \textbf{MN40} & \textbf{Congress} & \textbf{Wallmart} & \textbf{Senate} & \textbf{House} &\begin{tabular}[c]{@{}c@{}}\textbf{Avg}\\\textbf{Rank}\end{tabular} \\
    \midrule
    \midrule
    Linearized HGNN & $42.7_{\pm 8.4}$ & $34.3_{\pm 9.4}$ & $51.7_{\pm 6.6}$ & $49.6_{\pm 6.9}$ & $68.0_{\pm 6.0}$ & $94.5_{\pm 0.5}$ & $83.7_{\pm 5.0}$ & $29.0_{\pm 6.4}$ & $55.8_{\pm 5.1}$ & $57.8_{\pm 6.5}$ & 4.5\\
    \addlinespace
    No TCS, No RAP & $44.3_{\pm 7.2}$ & $35.0_{\pm 7.4}$ & $52.8_{\pm 6.0}$ & $48.8_{\pm 6.8}$ & $64.6_{\pm 7.5}$ & $97.6_{\pm 0.5}$ & $\mathbf{88.5_{\pm 2.3}}$ & $22.3_{\pm 6.6}$ & $\underline{71.8_{\pm 5.1}}$ & $69.0_{\pm 4.6}$ & 3.7 \\
    \addlinespace
    No TCS & $46.9_{\pm 5.5}$ & $37.8_{\pm 6.9}$ & $\underline{53.3_{\pm 6.0}}$ & $\underline{52.5_{\pm 5.2}}$ & $\mathbf{69.8_{\pm 7.5}}$ & $\mathbf{97.8 _{\pm 0.2}}$ & $87.0_{\pm 2.5}$ & $26.6_{\pm 5.0}$ & $67.2_{\pm 10.1}$ & $\underline{71.6_{\pm 5.9}}$ & 2.7\\
    \addlinespace
    No RAP & $\underline{50.6_{\pm 8.8}}$ & $\underline{48.6_{\pm 4.4}}$ & $\mathbf{62.6_{\pm 4.2}}$ & $\mathbf{60.0_{\pm 5.8}}$ & $64.7_{\pm 4.9}$ & $\underline{97.7_{\pm 0.3}}$ & $\underline{88.4_{\pm 2.5}}$ & $\underline{40.6_{\pm 4.9}}$ & $\mathbf{73.8_{\pm 5.2}}$ & $71.4_{\pm 3.9}$ &\underline{2.0}\\
    \midrule
    \midrule
    \textbf{\method (proposed})& $\mathbf{51.9_{\pm 10.1}}$ & $\mathbf{49.1_{\pm 4.8}}$ & $\mathbf{62.6_{\pm 3.9}}$ & $\mathbf{60.0_{\pm 6.2}}$ & $\underline{68.6_{\pm 4.8}}$ & $97.6_{\pm 0.3}$ & $87.0_{\pm 4.8}$ & $\mathbf{43.9_{\pm 3.1}}$ & $70.4_{\pm 10.0}$ & $\mathbf{73.2_{\pm 6.3}}$ &\textbf{1.7}\\
    \bottomrule
  \end{tabular}
  }
\end{table}

%% file: sections/050discussion.tex
\section{Conclusion}
\label{sec:conclusion}

In this work, we propose \method (\methodlong), a parameter-free model for few-shot node classification on hypergraphs. By reformulating existing HNNs into a unified linear framework with a tractable closed-form weight solution and redundancy-aware propagation, \method achieves strong generalization, fast inference, and interpretable representations without iterative training. Extensive experiments demonstrate its superior accuracy and scalability. One limitation of our work is that \method is specifically designed for node classification and is not tailored to other hypergraph tasks such as hyperedge prediction, local clustering, or hyperedge disambiguation. In future work, we plan to extend our framework to support a broader range of hypergraph learning tasks by designing a more general-purpose and efficient linear HNN architecture.

%% file: checklist.tex
\newpage
\section*{NeurIPS Paper Checklist}

\begin{enumerate}

\item {\bf Claims}
    \item[] Question: Do the main claims made in the abstract and introduction accurately reflect the paper's contributions and scope?
    \item[] Answer: \answerYes{} 
    \item[] Justification: To the best of our knowledge, our work does so.
    \item[] Guidelines:
    \begin{itemize}
        \item The answer NA means that the abstract and introduction do not include the claims made in the paper.
        \item The abstract and/or introduction should clearly state the claims made, including the contributions made in the paper and important assumptions and limitations. A No or NA answer to this question will not be perceived well by the reviewers. 
        \item The claims made should match theoretical and experimental results, and reflect how much the results can be expected to generalize to other settings. 
        \item It is fine to include aspirational goals as motivation as long as it is clear that these goals are not attained by the paper. 
    \end{itemize}

\item {\bf Limitations}
    \item[] Question: Does the paper discuss the limitations of the work performed by the authors?
    \item[] Answer: \answerYes{} 
    \item[] Justification: To the best of our knowledge, our work does so.
    \item[] Guidelines:
    \begin{itemize}
        \item The answer NA means that the paper has no limitation while the answer No means that the paper has limitations, but those are not discussed in the paper. 
        \item The authors are encouraged to create a separate "Limitations" section in their paper.
        \item The paper should point out any strong assumptions and how robust the results are to violations of these assumptions (e.g., independence assumptions, noiseless settings, model well-specification, asymptotic approximations only holding locally). The authors should reflect on how these assumptions might be violated in practice and what the implications would be.
        \item The authors should reflect on the scope of the claims made, e.g., if the approach was only tested on a few datasets or with a few runs. In general, empirical results often depend on implicit assumptions, which should be articulated.
        \item The authors should reflect on the factors that influence the performance of the approach. For example, a facial recognition algorithm may perform poorly when image resolution is low or images are taken in low lighting. Or a speech-to-text system might not be used reliably to provide closed captions for online lectures because it fails to handle technical jargon.
        \item The authors should discuss the computational efficiency of the proposed algorithms and how they scale with dataset size.
        \item If applicable, the authors should discuss possible limitations of their approach to address problems of privacy and fairness.
        \item While the authors might fear that complete honesty about limitations might be used by reviewers as grounds for rejection, a worse outcome might be that reviewers discover limitations that aren't acknowledged in the paper. The authors should use their best judgment and recognize that individual actions in favor of transparency play an important role in developing norms that preserve the integrity of the community. Reviewers will be specifically instructed to not penalize honesty concerning limitations.
    \end{itemize}

\item {\bf Theory assumptions and proofs}
    \item[] Question: For each theoretical result, does the paper provide the full set of assumptions and a complete (and correct) proof?
    \item[] Answer: \answerYes{} 
    \item[] Justification: To the best of our knowledge, our work does so.
    \item[] Guidelines:
    \begin{itemize}
        \item The answer NA means that the paper does not include theoretical results. 
        \item All the theorems, formulas, and proofs in the paper should be numbered and cross-referenced.
        \item All assumptions should be clearly stated or referenced in the statement of any theorems.
        \item The proofs can either appear in the main paper or the supplemental material, but if they appear in the supplemental material, the authors are encouraged to provide a short proof sketch to provide intuition. 
        \item Inversely, any informal proof provided in the core of the paper should be complemented by formal proofs provided in appendix or supplemental material.
        \item Theorems and Lemmas that the proof relies upon should be properly referenced. 
    \end{itemize}

    \item {\bf Experimental result reproducibility}
    \item[] Question: Does the paper fully disclose all the information needed to reproduce the main experimental results of the paper to the extent that it affects the main claims and/or conclusions of the paper (regardless of whether the code and data are provided or not)?
    \item[] Answer: \answerYes{} 
    \item[] Justification: To the best of our knowledge, our work does so.
    \item[] Guidelines:
    \begin{itemize}
        \item The answer NA means that the paper does not include experiments.
        \item If the paper includes experiments, a No answer to this question will not be perceived well by the reviewers: Making the paper reproducible is important, regardless of whether the code and data are provided or not.
        \item If the contribution is a dataset and/or model, the authors should describe the steps taken to make their results reproducible or verifiable. 
        \item Depending on the contribution, reproducibility can be accomplished in various ways. For example, if the contribution is a novel architecture, describing the architecture fully might suffice, or if the contribution is a specific model and empirical evaluation, it may be necessary to either make it possible for others to replicate the model with the same dataset, or provide access to the model. In general. releasing code and data is often one good way to accomplish this, but reproducibility can also be provided via detailed instructions for how to replicate the results, access to a hosted model (e.g., in the case of a large language model), releasing of a model checkpoint, or other means that are appropriate to the research performed.
        \item While NeurIPS does not require releasing code, the conference does require all submissions to provide some reasonable avenue for reproducibility, which may depend on the nature of the contribution. For example
        \begin{enumerate}
            \item If the contribution is primarily a new algorithm, the paper should make it clear how to reproduce that algorithm.
            \item If the contribution is primarily a new model architecture, the paper should describe the architecture clearly and fully.
            \item If the contribution is a new model (e.g., a large language model), then there should either be a way to access this model for reproducing the results or a way to reproduce the model (e.g., with an open-source dataset or instructions for how to construct the dataset).
            \item We recognize that reproducibility may be tricky in some cases, in which case authors are welcome to describe the particular way they provide for reproducibility. In the case of closed-source models, it may be that access to the model is limited in some way (e.g., to registered users), but it should be possible for other researchers to have some path to reproducing or verifying the results.
        \end{enumerate}
    \end{itemize}

\item {\bf Open access to data and code}
    \item[] Question: Does the paper provide open access to the data and code, with sufficient instructions to faithfully reproduce the main experimental results, as described in supplemental material?
    \item[] Answer: \answerYes{} 
    \item[] Justification: To the best of our knowledge, our work does so.
    \item[] Guidelines:
    \begin{itemize}
        \item The answer NA means that paper does not include experiments requiring code.
        \item Please see the NeurIPS code and data submission guidelines (\url{https://nips.cc/public/guides/CodeSubmissionPolicy}) for more details.
        \item While we encourage the release of code and data, we understand that this might not be possible, so “No” is an acceptable answer. Papers cannot be rejected simply for not including code, unless this is central to the contribution (e.g., for a new open-source benchmark).
        \item The instructions should contain the exact command and environment needed to run to reproduce the results. See the NeurIPS code and data submission guidelines (\url{https://nips.cc/public/guides/CodeSubmissionPolicy}) for more details.
        \item The authors should provide instructions on data access and preparation, including how to access the raw data, preprocessed data, intermediate data, and generated data, etc.
        \item The authors should provide scripts to reproduce all experimental results for the new proposed method and baselines. If only a subset of experiments are reproducible, they should state which ones are omitted from the script and why.
        \item At submission time, to preserve anonymity, the authors should release anonymized versions (if applicable).
        \item Providing as much information as possible in supplemental material (appended to the paper) is recommended, but including URLs to data and code is permitted.
    \end{itemize}

\item {\bf Experimental setting/details}
    \item[] Question: Does the paper specify all the training and test details (e.g., data splits, hyperparameters, how they were chosen, type of optimizer, etc.) necessary to understand the results?
    \item[] Answer: \answerYes{} 
    \item[] Justification: To the best of our knowledge, our work does so.
    \item[] Guidelines:
    \begin{itemize}
        \item The answer NA means that the paper does not include experiments.
        \item The experimental setting should be presented in the core of the paper to a level of detail that is necessary to appreciate the results and make sense of them.
        \item The full details can be provided either with the code, in appendix, or as supplemental material.
    \end{itemize}

\item {\bf Experiment statistical significance}
    \item[] Question: Does the paper report error bars suitably and correctly defined or other appropriate information about the statistical significance of the experiments?
    \item[] Answer: \answerYes{} 
    \item[] Justification: To the best of our knowledge, our work does so.
    \item[] Guidelines:
    \begin{itemize}
        \item The answer NA means that the paper does not include experiments.
        \item The authors should answer "Yes" if the results are accompanied by error bars, confidence intervals, or statistical significance tests, at least for the experiments that support the main claims of the paper.
        \item The factors of variability that the error bars are capturing should be clearly stated (for example, train/test split, initialization, random drawing of some parameter, or overall run with given experimental conditions).
        \item The method for calculating the error bars should be explained (closed form formula, call to a library function, bootstrap, etc.)
        \item The assumptions made should be given (e.g., Normally distributed errors).
        \item It should be clear whether the error bar is the standard deviation or the standard error of the mean.
        \item It is OK to report 1-sigma error bars, but one should state it. The authors should preferably report a 2-sigma error bar than state that they have a 96\% CI, if the hypothesis of Normality of errors is not verified.
        \item For asymmetric distributions, the authors should be careful not to show in tables or figures symmetric error bars that would yield results that are out of range (e.g. negative error rates).
        \item If error bars are reported in tables or plots, The authors should explain in the text how they were calculated and reference the corresponding figures or tables in the text.
    \end{itemize}

\item {\bf Experiments compute resources}
    \item[] Question: For each experiment, does the paper provide sufficient information on the computer resources (type of compute workers, memory, time of execution) needed to reproduce the experiments?
    \item[] Answer: \answerYes{} 
    \item[] Justification: To the best of our knowledge, our work does so.
    \item[] Guidelines:
    \begin{itemize}
        \item The answer NA means that the paper does not include experiments.
        \item The paper should indicate the type of compute workers CPU or GPU, internal cluster, or cloud provider, including relevant memory and storage.
        \item The paper should provide the amount of compute required for each of the individual experimental runs as well as estimate the total compute. 
        \item The paper should disclose whether the full research project required more compute than the experiments reported in the paper (e.g., preliminary or failed experiments that didn't make it into the paper). 
    \end{itemize}
    
\item {\bf Code of ethics}
    \item[] Question: Does the research conducted in the paper conform, in every respect, with the NeurIPS Code of Ethics \url{https://neurips.cc/public/EthicsGuidelines}?
    \item[] Answer: \answerYes{} 
    \item[] Justification: To the best of our knowledge, our work does so.
    \item[] Guidelines:
    \begin{itemize}
        \item The answer NA means that the authors have not reviewed the NeurIPS Code of Ethics.
        \item If the authors answer No, they should explain the special circumstances that require a deviation from the Code of Ethics.
        \item The authors should make sure to preserve anonymity (e.g., if there is a special consideration due to laws or regulations in their jurisdiction).
    \end{itemize}

\item {\bf Broader impacts}
    \item[] Question: Does the paper discuss both potential positive societal impacts and negative societal impacts of the work performed?
    \item[] Answer: \answerYes{} 
    \item[] Justification: To the best of our knowledge, our work does so.
    \item[] Guidelines:
    \begin{itemize}
        \item The answer NA means that there is no societal impact of the work performed.
        \item If the authors answer NA or No, they should explain why their work has no societal impact or why the paper does not address societal impact.
        \item Examples of negative societal impacts include potential malicious or unintended uses (e.g., disinformation, generating fake profiles, surveillance), fairness considerations (e.g., deployment of technologies that could make decisions that unfairly impact specific groups), privacy considerations, and security considerations.
        \item The conference expects that many papers will be foundational research and not tied to particular applications, let alone deployments. However, if there is a direct path to any negative applications, the authors should point it out. For example, it is legitimate to point out that an improvement in the quality of generative models could be used to generate deepfakes for disinformation. On the other hand, it is not needed to point out that a generic algorithm for optimizing neural networks could enable people to train models that generate Deepfakes faster.
        \item The authors should consider possible harms that could arise when the technology is being used as intended and functioning correctly, harms that could arise when the technology is being used as intended but gives incorrect results, and harms following from (intentional or unintentional) misuse of the technology.
        \item If there are negative societal impacts, the authors could also discuss possible mitigation strategies (e.g., gated release of models, providing defenses in addition to attacks, mechanisms for monitoring misuse, mechanisms to monitor how a system learns from feedback over time, improving the efficiency and accessibility of ML).
    \end{itemize}
    
\item {\bf Safeguards}
    \item[] Question: Does the paper describe safeguards that have been put in place for responsible release of data or models that have a high risk for misuse (e.g., pretrained language models, image generators, or scraped datasets)?
    \item[] Answer: \answerNA{} 
    \item[] Justification: Our work poses no such risks.
    \item[] Guidelines:
    \begin{itemize}
        \item The answer NA means that the paper poses no such risks.
        \item Released models that have a high risk for misuse or dual-use should be released with necessary safeguards to allow for controlled use of the model, for example by requiring that users adhere to usage guidelines or restrictions to access the model or implementing safety filters. 
        \item Datasets that have been scraped from the Internet could pose safety risks. The authors should describe how they avoided releasing unsafe images.
        \item We recognize that providing effective safeguards is challenging, and many papers do not require this, but we encourage authors to take this into account and make a best faith effort.
    \end{itemize}

\item {\bf Licenses for existing assets}
    \item[] Question: Are the creators or original owners of assets (e.g., code, data, models), used in the paper, properly credited and are the license and terms of use explicitly mentioned and properly respected?
    \item[] Answer: \answerYes{} 
    \item[] Justification: To the best of our knowledge, our work does so.
    \item[] Guidelines:
    \begin{itemize}
        \item The answer NA means that the paper does not use existing assets.
        \item The authors should cite the original paper that produced the code package or dataset.
        \item The authors should state which version of the asset is used and, if possible, include a URL.
        \item The name of the license (e.g., CC-BY 4.0) should be included for each asset.
        \item For scraped data from a particular source (e.g., website), the copyright and terms of service of that source should be provided.
        \item If assets are released, the license, copyright information, and terms of use in the package should be provided. For popular datasets, \url{paperswithcode.com/datasets} has curated licenses for some datasets. Their licensing guide can help determine the license of a dataset.
        \item For existing datasets that are re-packaged, both the original license and the license of the derived asset (if it has changed) should be provided.
        \item If this information is not available online, the authors are encouraged to reach out to the asset's creators.
    \end{itemize}

\item {\bf New assets}
    \item[] Question: Are new assets introduced in the paper well documented and is the documentation provided alongside the assets?
    \item[] Answer: \answerYes{} 
    \item[] Justification: To the best of our knowledge, our work does so.
    \item[] Guidelines:
    \begin{itemize}
        \item The answer NA means that the paper does not release new assets.
        \item Researchers should communicate the details of the dataset/code/model as part of their submissions via structured templates. This includes details about training, license, limitations, etc. 
        \item The paper should discuss whether and how consent was obtained from people whose asset is used.
        \item At submission time, remember to anonymize your assets (if applicable). You can either create an anonymized URL or include an anonymized zip file.
    \end{itemize}

\item {\bf Crowdsourcing and research with human subjects}
    \item[] Question: For crowdsourcing experiments and research with human subjects, does the paper include the full text of instructions given to participants and screenshots, if applicable, as well as details about compensation (if any)? 
    \item[] Answer: \answerNA{} 
    \item[] Justification: The paper does not involve crowdsourcing nor research with human subjects.
    \item[] Guidelines:
    \begin{itemize}
        \item The answer NA means that the paper does not involve crowdsourcing nor research with human subjects.
        \item Including this information in the supplemental material is fine, but if the main contribution of the paper involves human subjects, then as much detail as possible should be included in the main paper. 
        \item According to the NeurIPS Code of Ethics, workers involved in data collection, curation, or other labor should be paid at least the minimum wage in the country of the data collector. 
    \end{itemize}

\item {\bf Institutional review board (IRB) approvals or equivalent for research with human subjects}
    \item[] Question: Does the paper describe potential risks incurred by study participants, whether such risks were disclosed to the subjects, and whether Institutional Review Board (IRB) approvals (or an equivalent approval/review based on the requirements of your country or institution) were obtained?
    \item[] Answer: \answerNA{} 
    \item[] Justification: The paper does not involve crowdsourcing nor research with human subjects.
    \item[] Guidelines:
    \begin{itemize}
        \item The answer NA means that the paper does not involve crowdsourcing nor research with human subjects.
        \item Depending on the country in which research is conducted, IRB approval (or equivalent) may be required for any human subjects research. If you obtained IRB approval, you should clearly state this in the paper. 
        \item We recognize that the procedures for this may vary significantly between institutions and locations, and we expect authors to adhere to the NeurIPS Code of Ethics and the guidelines for their institution. 
        \item For initial submissions, do not include any information that would break anonymity (if applicable), such as the institution conducting the review.
    \end{itemize}

\item {\bf Declaration of LLM usage}
    \item[] Question: Does the paper describe the usage of LLMs if it is an important, original, or non-standard component of the core methods in this research? Note that if the LLM is used only for writing, editing, or formatting purposes and does not impact the core methodology, scientific rigorousness, or originality of the research, declaration is not required.
    \item[] Answer: \answerNA{} 
    \item[] Justification: The core method development in this research does not involve LLMs.
    \item[] Guidelines:
    \begin{itemize}
        \item The answer NA means that the core method development in this research does not involve LLMs as any important, original, or non-standard components.
        \item Please refer to our LLM policy (\url{https://neurips.cc/Conferences/2025/LLM}) for what should or should not be described.
    \end{itemize}

\end{enumerate}

%% file: appendices/appendix.tex
\newpage
\appendix

\input{appendices/A}
\input{appendices/B}
\input{appendices/C}
\input{appendices/D}
\input{appendices/E}
\input{appendices/F}
\input{appendices/G}



%% file: appendices/A.tex
\section{Proofs}
\input{appendices/A1}
\input{appendices/A2}
\input{appendices/A3}

%% file: appendices/A1.tex
\subsection{Proof of lemma 1}
\label{appendix:proof-lemma-1}

\begin{equation}
    \frac{\partial\mathcal{L_\mathrm{SSE}}}{\partial\W}
    = -2(\Dtrn\P\X)^\top(\Dtrn\Y-\Dtrn\Yhat)
\end{equation}
where $\Yhat = \P\X\W$. The optimal $\W_*$ minimizes $\mathcal{L_\mathrm{SSE}}$:
\begin{align}
    & -2(\Dtrn\P\X)^\top(\Dtrn\Y-\Dtrn\P\X\W_*) = 0 \\
    & (\Dtrn\P\X)^\top\Dtrn\P\X\W_*=(\Dtrn\P\X)^\top\Dtrn\Y \\
    & \left((\P\X)^\top\Dtrn(\P\X)\right)\W_*=(\P\X)^\top\Dtrn\Y \\
    & \W_*=\left((\P\X)^\top\Dtrn(\P\X)\right)^\dagger(\P\X)^\top\Dtrn\Y
\end{align}

%% file: appendices/A2.tex
\subsection{Proof of theorem 1}
\label{appendix:proof-theorem-1}

Without loss of generality, we can set the first $k \cdot c$ diagonal elements in $\Dtrn$ to be 1, and they are ordered with their labels. Let $\Dtrn'=[\I_{kc} \ \mathbf{0}] \in \mathbb{R}^{kc \times |\mathcal{V}|}$. We can easily accept $\Dtrn'^\top\Dtrn'=\Dtrn = (\Dtrn)^2$ by definition.

We first prove the following four lemmas for proving the main theorem:
\begin{lemma}
    By definition, the following holds:
    \begin{equation}
        \Dtrn'^\top\Dtrn'=\Dtrn = (\Dtrn)^2
    \end{equation}
\end{lemma}
\begin{proof}
    The proof is straightforward.
\end{proof}

\begin{lemma}
    Under the assumptions, $(\Dtrn'(\P\X))(\Dtrn'(\P\X))^\top$ can be expressed as:
    \begin{equation}
         (\Dtrn'(\P\X))(\Dtrn'(\P\X))^\top = (1-2\epsilon)(\I_c \otimes \J_k) + \epsilon(\I_{kc} + \J_{kc})
    \end{equation}
    where $\J$ is an all-one matrix, i.e. $\J_{kc}=\mathbf{1}_{kc} \mathbf{1}_{kc}^\top$, and $\otimes$ denotes Kronecker product. 
\end{lemma}
\begin{proof}
    The proof is straightforward.
\end{proof}

\begin{lemma}
    When $\epsilon>0$,
    $(1-2\epsilon)(\I_c \otimes \J_k) + \epsilon(\I_{kc} + \J_{kc})$ has inverse matrix as follows:
    \begin{equation}
        \left((1-2\epsilon)(\I_c \otimes \J_k) + \epsilon(\I_{kc} + \J_{kc})\right)^{-1} 
        =\frac{1}{\lambda_1}\M_1+\frac{1}{\lambda_2}\M_2+\frac{1}{\lambda_3}\M_3
    \end{equation}
    where $\lambda_1 = \epsilon, 
    \lambda_2 = (1 - 2\epsilon)k + \epsilon, 
    \lambda_3 = k(1 - 2\epsilon + \epsilon c) + \epsilon, \M_1 = \I_{kc} - \frac{1}{k}(\I_c \otimes  \J_k), \M_2 = \frac{1}{k}(\I_c \otimes  \J_k) - \frac{1}{kc} \J_{kc}, \M_3 = \frac{1}{kc} \J_{kc}$
\end{lemma}
\begin{proof}
We consider the matrix $\M = (1 - 2\epsilon)(\I_c \otimes \J_k) + \epsilon (\I_{kc} + \J_{kc})$.
We define three mutually orthogonal projection matrices:
\begin{equation}
    \M_1 = \I_{kc} - \frac{1}{k}(\I_c \otimes  \J_k), \quad \M_2 = \frac{1}{k}(\I_c \otimes  \J_k) - \frac{1}{kc} \J_{kc}, \quad \M_3 = \frac{1}{kc} \J_{kc}
\end{equation}
It is easily verified that these satisfy
\begin{equation}
    \M_i^2 = \M_i, \quad \M_i \M_j = 0 \ (i \neq j), \quad  \M_1 + \M_2 + \M_3 = \I_{kc}.
\end{equation}
We now compute the action of $\M$ on each subspace:
\begin{align*}
\M \M_1 &= \varepsilon \M_1, \\
\M \M_2 &= \left((1 - 2\varepsilon)k + \varepsilon\right) \M_2, \\
\M \M_3 &= \left(k(1 - 2\varepsilon + \varepsilon c) + \varepsilon\right) \M_3.
\end{align*}
Thus, $\M$ admits the spectral decomposition
\begin{equation}
    \M = \lambda_1 \M_1 + \lambda_2 \M_2 + \lambda_3 \M_3,
\end{equation}
where
\begin{equation}
    \lambda_1 = \epsilon, \quad 
    \lambda_2 = (1 - 2\epsilon)k + \epsilon, \quad 
    \lambda_3 = k(1 - 2\epsilon + \epsilon c) + \epsilon.
\end{equation}
Since $\epsilon > 0$, all eigenvalues are strictly positive, so $\M$ is invertible. The inverse is given by
\begin{equation}
    \M^{-1} = \lambda_1^{-1} \M_1 + \lambda_2^{-1} \M_2 + \lambda_3^{-1} \M_3.
\end{equation}
\end{proof}

\begin{lemma}
    Under the small $\epsilon \ll 1$, the following holds:
    \begin{equation}
        ((\Dtrn'(\P\X))(\Dtrn'(\P\X))^\top)^{-2} \approx \frac{1}{\epsilon}((\Dtrn'(\P\X))(\Dtrn'(\P\X))^\top)^{-1}
    \end{equation}
\end{lemma}
\begin{proof}
    \begin{align}
        &((\Dtrn'(\P\X))(\Dtrn'(\P\X))^\top)^{-1} 
        =\frac{1}{\lambda_1}\M_1+\frac{1}{\lambda_2}\M_2+\frac{1}{\lambda_3}\M_3 \\
        &((\Dtrn'(\P\X))(\Dtrn'(\P\X))^\top)^{-2} 
        =\frac{1}{\lambda_1^2}\M_1+\frac{1}{\lambda_2^2}\M_2+\frac{1}{\lambda_3^2}\M_3
    \end{align}
    When $\epsilon \ll 1$, $\lambda_1 = \epsilon, \lambda_2 \approx k, \lambda_3 \approx k$. Therefore, 
    \begin{align}
        &((\Dtrn'(\P\X))(\Dtrn'(\P\X))^\top)^{-1} 
        \approx \frac{1}{\epsilon}\M_1 \\
        &((\Dtrn'(\P\X))(\Dtrn'(\P\X))^\top)^{-2} 
        \approx \frac{1}{\epsilon^2}\M_1
    \end{align}
    Thus, $((\Dtrn'(\P\X))(\Dtrn'(\P\X))^\top)^{-2} \approx \frac{1}{\epsilon}((\Dtrn'(\P\X))(\Dtrn'(\P\X))^\top)^{-1}$ with small $\epsilon \ll 1$.
\end{proof}

The rest of this section proves the theorem based on above lemmas. First, we can reformulate $\K^\dagger=((\P\X)^\top \Dtrn (\P\X))^\dagger$ as:
\begin{equation}
    \K^\dagger
    = ((\P\X)^\top \Dtrn'^\top\Dtrn' (\P\X))^\dagger
    = ((\Dtrn'(\P\X))^\top (\Dtrn'(\P\X)))^\dagger
\end{equation}
by Lemma 4.

By definition of Moore–Penrose pseudoinverse matrix,
\begin{equation}
    ((\Dtrn'(\P\X))^\top (\Dtrn'(\P\X)))^\dagger= (\Dtrn'(\P\X))^\top\left((\Dtrn'(\P\X))(\Dtrn'(\P\X))^\top\right)^{-2}(\Dtrn'(\P\X))
\end{equation}
where $(\Dtrn'(\P\X))(\Dtrn'(\P\X))^\top$ has inverse matrix by Lemma 6.

By Lemma 7, 
\begin{equation}
    \K^\dagger
    \approx \frac{1}{\epsilon}(\Dtrn'(\P\X))^\top\left((\Dtrn'(\P\X))(\Dtrn'(\P\X))^\top\right)^{-1}(\Dtrn'(\P\X))
\end{equation}
Therefore, 
\begin{equation}
    \W^*=\K^\dagger\Z^\top\Ytrn
    \approx \frac{1}{\epsilon}(\Dtrn'(\P\X))^\top\left((\Dtrn'(\P\X))(\Dtrn'(\P\X))^\top\right)^{-1}(\Dtrn'(\P\X))(\P\X)^\top\Dtrn\Y
\end{equation}
By Lemma 4,
\begin{align}
    &\frac{1}{\epsilon}(\Dtrn'(\P\X))^\top\left((\Dtrn'(\P\X))(\Dtrn'(\P\X))^\top\right)^{-1}(\Dtrn'(\P\X))(\P\X)^\top\Dtrn\Y \\
    &=\frac{1}{\epsilon}(\Dtrn'(\P\X))^\top\left((\Dtrn'(\P\X))(\Dtrn'(\P\X))^\top\right)^{-1}(\Dtrn'(\P\X))(\P\X)^\top\Dtrn'^\top\Dtrn'\Dtrn\Y \\ 
    &= \frac{1}{\epsilon}(\Dtrn'(\P\X))^\top\left((\Dtrn'(\P\X))(\Dtrn'(\P\X))^\top\right)^{-1}(\Dtrn'(\P\X))(\Dtrn'(\P\X))'^\top\Dtrn'\Dtrn\Y \\
    &= \frac{1}{\epsilon}(\Dtrn'(\P\X))^\top\Dtrn'\Dtrn\Y
    = \frac{1}{\epsilon}(\P\X)^\top\Dtrn'^\top\Dtrn'\Dtrn\Y
    = \frac{1}{\epsilon}(\P\X)^\top\Dtrn\Y
\end{align}
Therefore,
\begin{equation}
    \W^* \approx \frac{1}{\epsilon}(\P\X)^\top\Dtrn\Y
\end{equation}

%% file: appendices/A3.tex
\subsection{Proof of lemma 2}
\label{appendix:proof-lemma-2}

We analyze the diagonal entries of the $\A_1 = \Dvh \H \Def \H^\top \Dvh$. Specifically,
    \begin{equation}
        (\A_1)_{ii}=\frac{1}{\sqrt{d_{v_i}}}(\H\Def\H^\top)_{ii}\frac{1}{\sqrt{d_{v_i}}}= \frac{1}{d_{v_i}}(\H\Def\H^\top)_{ii}
    \end{equation}
    Since $\H_{ij}=1$ if and only if node $v_i$ belongs to hyperedge $e_j$, $(i,i)$-th entry of $\H\Def\H^\top$ corresponds to the sum of edge-normalized weights over all hyperedges incident to node $v_i$. Therefore, we obtain:
    \begin{equation}
        (\H\Def\H^\top)_{ii} =
        \sum_{e_j \in \mathcal{N}(v_i)} \frac{1}{d_{e_j}}
    \end{equation}

%% file: appendices/B.tex
\section{Linearization of five representative HNNs}
\label{appendix:linear}

\input{tables/original}

Table~\ref{tab:original-layers} presents the layer-wise formulations of five hypergraph neural networks (HNNs), where each MLP block is assumed to be a single-layer perceptron. By removing nonlinear activation functions, all weight matrices in a network can be merged into a single equivalent weight matrix, since multi-layer perceptrons reduce to a linear transformation without nonlinearity.

Let $\Dv$ and $\De$ be the diagonal degree matrices of nodes and hyperedges, respectively.
We define the variants of degree matrices as follows:
\[
(\tilde{\D}_\mathrm{e})_{ii} = d_{e_i}^{-1} \sum_{v_j \in \mathcal{N}(e_i)} d_{v_j},
\quad
(\D_{\mathrm{v},\alpha})_{ii} = \sum_{e_j \in \mathcal{N}(v_i)} d_{e_j}^{\alpha},
\quad
(\D_{\mathrm{e},\beta})_{ii} = \sum_{v_j \in \mathcal{N}(e_i)} d_{v_j}^{\beta}.
\]
We denote by $\sigma$ a nonlinear function (e.g., ReLU), and by $\tilde{\W} = (1 - \beta)\I + \beta \W$ a convex combination of the identity and a learnable weight matrix.
Since $\W$ is a free parameter being updated during the training, we can safely replace $\tilde{\W}$ with $\W$ under linearization.


%% file: tables/original.tex
\begin{table}[t]
  \caption{Layer-wise formulations of five HNNs, assuming single-layer MLPs.}
  \label{tab:original-layers}
  \centering
  \resizebox{\textwidth}{!}{
  \begin{tabular}{l|l}
    \toprule
    Method & Each layer \\
    \midrule
    HGNN & $\X^{(l)} = \sigma\left(\Dvh \H \Def \H^\top \Dvh \X^{(l-1)} \W \right)$ \\
    HNHN & $\X^{(l)} = \sigma\left(\D_{\mathrm{v},\alpha}^{-1} \H \De^{\alpha} \sigma\left(\D_{\mathrm{e},\beta}^{-1} \H^\top \Dv^{\beta} \X^{(l-1)} \W\right) \W \right)$ \\
    UniGCNII & $\X^{(l)} = \sigma\left(\left((1-\alpha)\Dvf \H \tilde{\D}_\mathrm{e}^{-1} \H^\top \X^{(l-1)} + \alpha \X^{(0)} \right)\tilde{\W} \right)$ \\
    AllDeepSet & $\X^{(l)} = \sigma\left(\Dvf \H \sigma\left(\Def \H^\top \sigma\left(\X^{(l-1)} \W \right) \W \right) \W \right)$ \\
    ED-HNN & $\X^{(l)} = \sigma\left(\left((1-\alpha) \sigma\left(\Dvf \H \sigma\left(\Def \H^\top \sigma\left(\X^{(l-1)} \W \right) \W \right) \W \right) + \alpha \X^{(0)} \right) \W \right)$ \\
    \bottomrule
  \end{tabular}
  }
\end{table}

%% file: appendices/C.tex
\section{Error bounds in Theorem~\ref{theorem:approximation}}

We present the relative Frobenius norm error, confirming it remains sufficiently low as $\epsilon$ decreases.
In TCS, our approximation is given as follows:
\begin{equation}
    \frac{1}{\lambda_1^2}\M_1+\frac{1}{\lambda_2^2}\M_2+\frac{1}{\lambda_3^2}\M_3
    \approx
    \frac{1}{\epsilon}\left(
    \frac{1}{\lambda_1}\M_1+\frac{1}{\lambda_2}\M_2+\frac{1}{\lambda_3}\M_3
    \right)
\end{equation}
where $\lambda_1 = \epsilon, 
\lambda_2 = (1 - 2\epsilon)k + \epsilon, \lambda_3 = k(1 - 2\epsilon + \epsilon c) + \epsilon, \M_1 = \I_{kc} - \frac{1}{k}(\I_c \otimes  \J_k), \M_2 = \frac{1}{k}(\I_c \otimes  \J_k) - \frac{1}{kc} \J_{kc}, \M_3 = \frac{1}{kc} \J_{kc}$. $\J$ denotes an all-one matrix and $\otimes$ is the Kronecker product.

To evaluate the approximation quality, we compute the relative Frobenius norm error as follows:
\begin{equation}
    \frac{\left\vert \frac{1}{\lambda_1^2}\M_1+\frac{1}{\lambda_2^2}\M_2+\frac{1}{\lambda_3^2}\M_3 - \frac{1}{\epsilon}\left(
    \frac{1}{\lambda_1}\M_1+\frac{1}{\lambda_2}\M_2+\frac{1}{\lambda_3}\M_3
    \right)\right\vert_F}{\left\vert\frac{1}{\lambda_1^2}\M_1+\frac{1}{\lambda_2^2}\M_2+\frac{1}{\lambda_3^2}\M_3\right\vert_F}
\end{equation}

For a representative case with $k=5$ and $c=10$, the relative errors are summarized in Table~\ref{tab:epsilon_error}.
\begin{table}[h]
\centering
\caption{Relative error for different values of $\epsilon$.}
\label{tab:epsilon_error}
\small
\renewcommand{\arraystretch}{1.1}
\setlength{\tabcolsep}{6pt}
\begin{tabular}{cc}
\toprule
$\epsilon$ & Relative Error \\
\midrule
0.1 & 1.14\% \\
0.01 & 0.10\% \\
0.001 & 0.01\% \\
\bottomrule
\end{tabular}
\end{table}

These results indicate that the approximation becomes increasingly accurate as $\epsilon$ decreases. This supports the validity of our approach, particularly in the regime where $\epsilon$ is sufficiently small.

%% file: appendices/D.tex
 

\section{Row-normalized adjacency matrices}
\label{appendix:row-norm}

Eq.~\eqref{eq:normalized-adj-3} and Eq.~\eqref{eq:normalized-adj-4} present the row-normalized adjacency matrices $\hat{\A}_1$ and $\hat{\A}_2$. The corresponding residual self-information terms, $\RSI(\hat{\A}_1)$ and $\RSI(\hat{\A}_2)$, are shown to coincide with those obtained under symmetric normalization. This equivalence is formally established in Lemma~\ref{lemma:rsa3} and Lemma~\ref{lemma:rsa4}.
\begin{align}
    &\hat{\A}_1 = \Dvf \H (\De-\I)^{-1} \H^\top 
    \label{eq:normalized-adj-3} \\
    &\hat{\A}_2 = \A_1^*
    \left((\Dv-\I)^{-1}\Dv\right)\A_1^*. \label{eq:normalized-adj-4}
\end{align}

\begin{lemma}
Given $\hat{\A}_1$ in Eq. \eqref{eq:normalized-adj-3}, $\RSI(\hat{\A}_1)$ is given by
\begin{equation}
\label{eq:rsi1}
\textstyle (\RSI(\hat{\A}_1))_{ii} = d_{v_i}^{-1} \left(\sum_{e_j \in \mathcal{N}(v_i)} (d_{e_j} - 1)^{-1} \right),
\end{equation}
where $d_x$ denotes the degree of node $x$ or the number of nodes in hyperedge $x$, based on the type of $x$, and $\mathcal{N}(v_i)$ denotes the set of hyperedges incident to node $v_i$.
\label{lemma:rsa3}
\end{lemma}

\begin{proof}
    We analyze the diagonal entries of the $\A_1 = \Dvf \H \Def \H^\top $. Specifically,
    \begin{equation}
        (\A_1)_{ii}=\frac{1}{d_{v_i}}(\H\Def\H^\top)_{ii}
    \end{equation}
    Since $\H_{ij}=1$ if and only if node $v_i$ belongs to hyperedge $e_j$, $(i,i)$-th entry of $\H\Def\H^\top$ corresponds to the sum of edge-normalized weights over all hyperedges incident to node $v_i$. Therefore, we obtain:
    \begin{equation}
        (\H\Def\H^\top)_{ii} =
        \sum_{e_j \in \mathcal{N}(v_i)} \frac{1}{d_{e_j}}
    \end{equation}
\end{proof}

\begin{lemma}
Given $\mathbf{A}_1^* = \hat{\mathbf{A}}_1 - \RSI(\hat{\mathbf{A}}_1)$ and $\hat{\A}_2$ in Eq. \eqref{eq:normalized-adj-4}, $\RSI(\hat{\A}_2)$ is given by
\begin{equation}
\label{eq:rsi4}
\textstyle
(\RSI(\hat{\A}_2))_{ii} = d_{v_i}^{-1} \left( \sum_{e_j \in \mathcal{N}(v_i)} (d_{e_j} - 1)^{-2} \left( \sum_{v_k \in \mathcal{N}(e_j) \setminus \{v_i\}} (d_{v_k} - 1)^{-1} \right) \right),
\end{equation}
where $d_x$ denotes the degree of node $x$ or the number of nodes in hyperedge $x$, based on the type of $x$, $\mathcal{N}(v_i)$ denotes the set of hyperedges incident to node $v_i$, and $\mathcal{N}(e_j)$ denotes the set of nodes incident to hyperedge $e_j$.
\label{lemma:rsa4}
\end{lemma}

\begin{proof}
    The proof follows by applying the same reasoning as in Lemma \ref{lemma:rsa3}.
\end{proof}

%% file: appendices/E.tex
\section{Possible approximations for deeper propagation}
\label{appendix:extensions}
Some datasets may require models with higher-hop propagations for better expressivity. Without approximation, the exact computation of RSI requires high computational cost. One possible approximation is to estimate the probability that a random walker returns to the starting node after  steps, where  is the number of hops that we aim to model. 
A simple pseudocode for this strategy is provided in Algorithm~\ref{alg:self-info}.



\begin{algorithm}[htp!]
\caption{Approximation via random walks}
\label{alg:self-info}
\begin{algorithmic}[1]
\Require Node $i$, walk length $l$
\Ensure Return probability of node $i$
\State $count \gets 0$
\For{each trial}
    \State $current \gets i$
    \For{$t = 1$ \textbf{to} $l$}
        \State Sample a hyperedge $e$ incident to $current$ uniformly at random
        \State Sample a node $j$ connected to $e$ uniformly at random
        \State $current \gets j$
    \EndFor
    \If{$current = i$} 
        \State $count \gets count + 1$
    \EndIf
\EndFor
\State \Return $count / (\text{number of trials})$
\end{algorithmic}
\end{algorithm}

Another possible approximation is to use Hutchinson's Estimator, which provides an unbiased stochastic estimate of the diagonal entries. The equation is given by:
\begin{equation}
    \operatorname{diag}(\A) \approx \frac{1}{m}\sum_{k=1}^m \mathbf{z}^{(k)} \odot \left(\A \mathbf{z}^{(k)}\right)
\end{equation}
where $m$ is the number of random probe vectors, each $\mathbf{z}^{(k)} \in \mathbb{R}^n$ is a random vector with entries independently sampled from the Rademacher distribution (i.e., each entry is $+1$ or $-1$ with equal probability), and $\odot$ denotes the element-wise (Hadamard) product. The expectation satisfies $\operatorname{diag}(\A) = \mathbb{E}[\mathbf{z} \odot \left(\A \mathbf{z}\right)]$.

Hutchinson's Estimator does not require explicit storage of the entire matrix $\A$; instead, it relies solely on matrix-vector multiplications with randomized vectors $\mathbf{z}$. This characteristic renders Hutchinson's Estimator particularly suitable and computationally efficient for hypergraph structures. We consider these approaches to be promising directions for extending ZEN to scenarios where deeper propagation may provide additional benefits.

%% file: appendices/F.tex
\section{Additional experiments}
In this section, we present additional experimental results omitted from the main paper due to space limitations.

\input{appendices/F1}

\input{appendices/F2}
\input{appendices/F3}

%% file: appendices/F1.tex
\subsection{Evaluation with increasing shots}
\label{appendix:k-node}

\method is a parameter-free model, which gives it a particular advantage in settings where labeled data is scarce and training is challenging. While \method was orginally designed for few-shot setting, our results in Table~\ref{tab:10nodes} and Table~\ref{tab:20nodes} show it performs strongly with more training samples as well. \method achieves top average ranks in both 10-shot and 20-shot settings, suggesting it scales well beyond few-shot scenarios.

\input{tables/acc_10nodes}
\input{tables/acc_20nodes}

%% file: tables/acc_10nodes.tex
\begin{table}[t]
  \caption{Classification accuracy (\%) for 10-shot node classification on real-world hypergraphs. We report the mean and standard deviation over 10 runs. Boldfaced letters indicate the best accuracy, and underlined letters indicate the second.
  \method achieves the highest average rank.}
  \label{tab:10nodes}
  \centering
  \setlength{\defaultaddspace}{0.5ex}
  \resizebox{\textwidth}{!}{
  \begin{tabular}{l|cccccccccc|c}
    \toprule
    
     \textbf{Methods} & \textbf{Cora} & \textbf{Citeseer} & \textbf{Pubmed} & \textbf{Cora\_CA} & \textbf{20News} &
     \textbf{MN40}&\textbf{Congress}&\textbf{Wallmart} &\textbf{Senate}&\textbf{House}
     &\begin{tabular}[c]{@{}c@{}}\textbf{Avg.}\\\textbf{Rank}\end{tabular}  \\
    \midrule
    \midrule
    HGNN & $59.13_{\pm 5.8}$ & $49.28_{\pm 3.1}$ & $60.53_{\pm 7.0}$ & $64.94_{\pm 4.4}$ & $\underline{75.54_{\pm 1.2}}$ & $94.90_{\pm 0.2}$ & $87.18_{\pm 0.9}$ & $46.85_{\pm 1.2}$ & $58.01_{\pm 2.5}$ & $64.30_{\pm 4.4}$ & $5.4$ \\
    \addlinespace
    HNHN & $42.47_{\pm 5.0}$ & $42.82_{\pm 2.8}$ & $58.92_{\pm 3.3}$ & $46.69_{\pm 3.5}$ & $51.23_{\pm 4.9}$ & $93.96_{\pm 0.6}$ & $47.77_{\pm 1.8}$ & $17.13_{\pm 2.2}$ & $69.74_{\pm 6.1}$ & $70.32_{\pm 7.6}$ & $8.0$ \\
    \addlinespace
    HCHA & $58.96_{\pm 5.4}$ & $49.65_{\pm 2.9}$ & $60.02_{\pm 5.6}$ & $64.28_{\pm 5.6}$ & $\mathbf{75.59_{\pm 1.3}}$ & $94.89_{\pm 0.2}$ & $87.23_{\pm 0.9}$ & $47.11_{\pm 1.7}$ & $57.89_{\pm 2.8}$ & $63.86_{\pm 4.4}$ & $5.7$ \\
    \addlinespace
    UniGCN & $\underline{60.87_{\pm 5.8}}$ & $51.13_{\pm 3.0}$ & $61.20_{\pm 6.9}$ & $\underline{66.87_{\pm 3.0}}$ & $73.69_{\pm 1.7}$ & $96.24_{\pm 0.2}$ & $\mathbf{91.35_{\pm 0.7}}$ & $45.38_{\pm 1.4}$ & $62.57_{\pm 2.5}$ & $70.27_{\pm 3.0}$ & $\underline{3.7}$ \\
    \addlinespace
    UniGCNII & $56.81_{\pm 5.1}$ & $49.18_{\pm 2.9}$ & $61.12_{\pm 6.4}$ & $63.17_{\pm 3.8}$ & ${70.91_{\pm 2.1}}$ & $\underline{97.03_{\pm 0.2}}$ & $87.72_{\pm 1.9}$ & $28.70_{\pm 2.5}$ & $\mathbf{75.14_{\pm 4.3}}$ & $\underline{73.65_{\pm 2.9}}$ & $5.1$ \\
    \addlinespace
    AllDeepSets & $57.31_{\pm 4.4}$ & $50.63_{\pm 2.6}$ & $61.39_{\pm 5.1}$ & $62.78_{\pm 4.1}$ & $62.15_{\pm 2.3}$ & $95.25_{\pm 0.3}$ & $72.63_{\pm 5.8}$ & $35.54_{\pm 2.7}$ & $69.53_{\pm 6.7}$ & $69.96_{\pm 3.1}$ & $6.3$ \\
    \addlinespace
    AllSetTransformer & $57.50_{\pm 5.0}$ & $\underline{54.07_{\pm 2.4}}$ & $\underline{63.41_{\pm 6.6}}$ & $\mathbf{67.63_{\pm 4.2}}$ & $72.67_{\pm 1.4}$ & $95.87_{\pm 0.1}$ & $83.09_{\pm 2.2}$ & $45.39_{\pm 3.7}$ & $73.13_{\pm 5.5}$ & $70.99_{\pm 4.5}$ & $3.9$ \\
    \addlinespace
    ED-HNN & $58.84_{\pm 4.3}$ & $51.12_{\pm 2.8}$ & $60.59_{\pm 7.0}$ & $64.73_{\pm 3.7}$ & $71.13_{\pm 2.2}$ & $96.30_{\pm 0.2}$ & $\underline{90.73_{\pm 1.8}}$ & $\mathbf{47.17_{\pm 2.3}}$ & $66.44_{\pm 9.3}$ & $62.49_{\pm 7.9}$ & $4.7$ \\
    \midrule
    \midrule 
    \textbf{ZEN (proposed)} & $\mathbf{61.44_{\pm 4.2}}$ & $\mathbf{59.17_{\pm 2.3}}$ & $\mathbf{68.76_{\pm 3.4}}$ & ${65.60_{\pm 3.9}}$ & $73.04_{\pm 2.4}$ & $\mathbf{97.92_{\pm 0.1}}$ & $86.54_{\pm 2.6}$ & $\underline{47.12_{\pm 3.2}}$ & $\underline{74.72_{\pm 1.9}}$ & $\mathbf{73.90_{\pm 6.4}}$ & $\textbf{2.2}$ \\  
    \bottomrule
  \end{tabular}
  }
\end{table}

%% file: tables/acc_20nodes.tex
\begin{table}[t]
  \caption{Classification accuracy (\%) for 20-shot node classification on real-world hypergraphs. We report the mean and standard deviation over 10 runs. Boldfaced letters indicate the best accuracy, and underlined letters indicate the second.
  \method achieves the highest average rank.}
  \label{tab:20nodes}
  \centering
  \setlength{\defaultaddspace}{0.5ex}
  \resizebox{\textwidth}{!}{
  \begin{tabular}{l|cccccccccc|c}
    \toprule
    
     \textbf{Methods} & \textbf{Cora} & \textbf{Citeseer} & \textbf{Pubmed} & \textbf{Cora\_CA} & \textbf{20News} &
     \textbf{MN40}&\textbf{Congress}&\textbf{Wallmart} &\textbf{Senate}&\textbf{House}
     &\begin{tabular}[c]{@{}c@{}}\textbf{Avg.}\\\textbf{Rank}\end{tabular}  \\
    \midrule
    \midrule
    HGNN & $62.46_{\pm 3.1}$ & $56.57_{\pm 1.9}$ & $69.65_{\pm 3.9}$ & $70.33_{\pm 1.9}$ & $\underline{76.66_{\pm 1.1}}$ & $95.07_{\pm 0.2}$ & $87.80_{\pm 0.5}$ & $49.46_{\pm 2.4}$ & $60.73_{\pm 3.2}$ & $64.83_{\pm 2.2}$ & $6.1$ \\
    \addlinespace
    HNHN & $48.28_{\pm 3.9}$ & $50.08_{\pm 2.3}$ & $62.60_{\pm 1.6}$ & $52.88_{\pm 2.9}$ & $52.69_{\pm 6.3}$ & $95.20_{\pm 0.5}$ & $49.92_{\pm 3.0}$ & $21.82_{\pm 1.8}$ & $73.51_{\pm 3.9}$ & $72.00_{\pm 4.1}$ & $8.0$ \\
    \addlinespace
    HCHA & $64.19_{\pm 3.9}$ & $56.81_{\pm 2.1}$ & $69.44_{\pm 3.3}$ & $70.04_{\pm 1.9}$ & $\mathbf{76.68_{\pm 1.0}}$ & $95.09_{\pm 0.1}$ & $87.91_{\pm 0.4}$ & $49.44_{\pm 2.2}$ & $61.07_{\pm 3.3}$ & $64.91_{\pm 2.3}$ & $5.7$ \\
    \addlinespace
    UniGCN & $\underline{65.44_{\pm 3.3}}$ & $58.47_{\pm 1.2}$ & $\underline{70.84_{\pm 2.6}}$ & $\underline{71.26_{\pm 1.8}}$ & $72.59_{\pm 2.1}$ & $96.63_{\pm 0.1}$ & $\mathbf{91.62_{\pm 0.3}}$ & $49.00_{\pm 2.1}$ & $62.04_{\pm 2.3}$ & $72.69_{\pm 0.9}$ & $3.9$ \\
    \addlinespace
    UniGCNII & $64.29_{\pm 2.9}$ & $56.97_{\pm 1.2}$ & $68.92_{\pm 3.4}$ & $68.55_{\pm 2.5}$ & $74.05_{\pm 1.5}$ & $\underline{97.56_{\pm 0.1}}$ & $88.31_{\pm 1.9}$ & $32.04_{\pm 1.5}$ & $\mathbf{77.59_{\pm 2.0}}$ & $\underline{76.05_{\pm 0.9}}$ & $4.4$ \\
    \addlinespace
    AllDeepSets & $62.37_{\pm 3.7}$ & $57.78_{\pm 1.8}$ & $67.10_{\pm 2.2}$ & $65.93_{\pm 1.3}$ & $67.60_{\pm 2.5}$ & $96.69_{\pm 0.3}$ & $77.65_{\pm 5.8}$ & $43.97_{\pm 2.4}$ & $73.05_{\pm 2.2}$ & $73.29_{\pm 2.5}$ & $6.4$ \\
    \addlinespace
    AllSetTransformer & $63.65_{\pm 1.6}$ & $\underline{59.49_{\pm 2.1}}$ & $69.70_{\pm 2.3}$ & $69.82_{\pm 2.2}$ & $74.22_{\pm 1.2}$ & $96.02_{\pm 0.2}$ & $89.20_{\pm 0.5}$ & $49.29_{\pm 1.8}$ & $\underline{76.14_{\pm 3.8}}$ & $75.58_{\pm 2.7}$ & $\underline{3.8}$ \\
    \addlinespace
    ED-HNN & $64.24_{\pm 3.0}$ & $57.27_{\pm 1.6}$ & $69.56_{\pm 3.4}$ & $69.29_{\pm 1.6}$ & $70.92_{\pm 1.9}$ & $96.69_{\pm 0.2}$ & $\underline{90.11_{\pm 2.1}}$ & $\underline{50.61_{\pm 3.3}}$ & $72.30_{\pm 6.7}$ & $71.27_{\pm 3.7}$ & $4.8$ \\
    \midrule
    \midrule 
    \textbf{ZEN (proposed)} & $\mathbf{67.87_{\pm 1.4}}$ & $\mathbf{64.00_{\pm 1.4}}$ & $\mathbf{71.32_{\pm 1.6}}$ & $\mathbf{71.99_{\pm 2.0}}$ & ${73.74_{\pm 1.2}}$ & $\mathbf{98.21_{\pm 0.1}}$ & $88.04_{\pm 2.9}$ & $\mathbf{51.30_{\pm 1.1}}$ & ${74.10_{\pm 5.2}}$ & $\mathbf{76.70_{\pm 0.4}}$ & $\textbf{2.0}$ \\  
    \bottomrule
  \end{tabular}
  }
\end{table}

%% file: appendices/F2.tex
\subsection{Evaluation against additional baselines}
\label{appendix:acc_linear}

We have additionally included eight baselines; three representative linear GNNs (SGC \cite{pmlr-v97-wu19e}, APPNP \cite{gasteiger2022predictpropagategraphneural}, and SSGC \cite{zhu2021simple}), three linearized hypergraph models based on UniGCNII \cite{huang2021unignnunifiedframeworkgraph}, AllDeepSets \cite{chien2022allsetmultisetfunctionframework}, and ED-HNN \cite{wang2023equivarianthypergraphdiffusionneural}, and two semi-supervised hypergraph models (LEGCN \cite{yang2022semi} and HyperND \cite{tudisco2021nonlinear}). LEGCN utilizes a line expansion approach to adapt hypergraphs to conventional GNN architectures, while HyperND introduces a diffusion-based mechanism for improved label propagation in hypergraphs. For GNNs, we applied clique expansion to convert the hypergraph into a pairwise graph. As shown in the Table~\ref{tab:additional-models}, \method demonstrates consistently strong performance, outperforming all linear GNN and linearized hypergraph baselines on the majority of datasets.

\input{tables/acc_linear}

%% file: tables/acc_linear.tex
\begin{table}[htp!]
  \caption{Comparison of classification accuracy (\%) with eight additional baselines for 5-shot node classification on real-world hypergraphs. We report the mean and standard deviation over 10 runs. Boldfaced letters indicate the best accuracy, and underlined letters indicate the second.
  \method achieves the highest average rank.}
  \label{tab:additional-models}
  \centering
  \setlength{\defaultaddspace}{0.5ex}
  \resizebox{\textwidth}{!}{
  \begin{tabular}{l|cccccccccc|c}
    \toprule
    
     \textbf{Methods} & \textbf{Cora} & \textbf{Citeseer} & \textbf{Pubmed} & \textbf{Cora\_CA} & \textbf{20News} &
     \textbf{MN40}&\textbf{Congress}&\textbf{Wallmart} &\textbf{Senate}&\textbf{House}
     &\begin{tabular}[c]{@{}c@{}}\textbf{Avg.}\\\textbf{Rank}\end{tabular}  \\
    \midrule
    \midrule
    SGC & $44.09_{\pm 9.8}$ & $40.42_{\pm 5.3}$ & $57.19_{\pm 5.5}$ & $50.79_{\pm 5.9}$ & $58.06_{\pm 5.1}$ & $91.58_{\pm 0.7}$ & $72.48_{\pm 3.1}$ & $24.11_{\pm 3.2}$ & $51.26_{\pm 2.7}$ & $52.27_{\pm 1.4}$ & $5.9$ \\
    \addlinespace
    APPNP & $\underline{45.77_{\pm 9.5}}$ & $39.55_{\pm 4.7}$ & $55.38_{\pm 7.1}$ & $50.04_{\pm 6.7}$ & $59.90_{\pm 5.8}$ & $92.75_{\pm 0.3}$ & $69.09_{\pm 6.8}$ & $24.70_{\pm 3.3}$ & $71.24_{\pm 5.4}$ & $70.00_{\pm 8.4}$ & $5.8$ \\
    \addlinespace
    SSGC & $42.60_{\pm 10.9}$ & $40.69_{\pm 4.7}$ & $57.07_{\pm 4.9}$ & $\underline{52.55_{\pm 6.7}}$ & $59.91_{\pm 5.2}$ & $93.75_{\pm 0.2}$ & $81.10_{\pm 2.4}$ & $\underline{26.56_{\pm 3.4}}$ & $\underline{73.00_{\pm 0.7}}$ & $71.74_{\pm 3.2}$ & $\underline{3.3}$ \\
    \addlinespace
    lin UniGCNII & $39.44_{\pm 7.4}$ & $40.00_{\pm 4.0}$ & $56.85_{\pm 4.9}$ & $49.75_{\pm 6.1}$ & $59.45_{\pm 5.2}$ & $96.58_{\pm 0.4}$ & $73.35_{\pm 3.1}$ & $17.91_{\pm 1.7}$ & $71.22_{\pm 10.8}$ & $68.23_{\pm 10.6}$ & $5.8$ \\
    \addlinespace
    lin AllDeepSets & $40.49_{\pm 8.2}$ & $\underline{40.76_{\pm 4.5}}$ & $56.13_{\pm 6.0}$ & $51.04_{\pm 5.8}$ & $\underline{66.62_{\pm 5.0}}$ & $96.62_{\pm 0.3}$ & $\mathbf{90.39_{\pm 1.5}}$ & $26.45_{\pm 3.0}$ & $60.95_{\pm 4.7}$ & $65.13_{\pm 13.1}$ & $4.3$ \\
    \addlinespace
    lin ED-HNN & $41.97_{\pm 5.6}$ & $39.64_{\pm 4.2}$ & $55.82_{\pm 5.4}$ & $47.37_{\pm 5.8}$ & $61.47_{\pm 6.0}$ & $\underline{97.28_{\pm 0.4}}$ & $78.35_{\pm 2.7}$ & $18.84_{\pm 1.6}$ & $71.10_{\pm 9.9}$ & $70.20_{\pm 9.3}$ & $5.2$ \\
    \addlinespace
    LEGCN & $37.59_{\pm 5.2}$ & $37.25_{\pm 3.8}$ & $\underline{58.11_{\pm 4.0}}$ & $37.59_{\pm 5.2}$ & $49.41_{\pm 4.3}$ & $93.27_{\pm 0.7}$ & $72.14_{\pm 3.4}$ & O.O.M & $71.25_{\pm 8.5}$ & $\underline{72.85_{\pm 7.0}}$ & $6.5$ \\
    \addlinespace
    HyperND & $39.09_{\pm 6.2}$ & $35.26_{\pm 4.8}$ & $56.54_{\pm 4.5}$ & $41.74_{\pm 4.6}$ & $54.70_{\pm 3.8}$ & $91.40_{\pm 0.5}$ & $73.57_{\pm 4.9}$ & $13.55_{\pm 1.7}$ & $\mathbf{73.88_{\pm 7.5}}$ & $72.84_{\pm 7.1}$ & $6.5$ \\
    \midrule
    \midrule 
    \textbf{ZEN (proposed)} & $\mathbf{51.85_{\pm 10.1}}$ & $\mathbf{49.08_{\pm 4.8}}$ & $\mathbf{62.62_{\pm 3.9}}$ & $\mathbf{60.04_{\pm 6.2}}$ & $\mathbf{68.57_{\pm 4.8}}$ & $\mathbf{97.63_{\pm 0.3}}$ & $\underline{86.96_{\pm 4.8}}$ & $\mathbf{43.88_{\pm 3.1}}$ & ${70.40_{\pm 10.0}}$ & $\mathbf{73.22_{\pm 6.3}}$ & $\textbf{1.7}$ \\  
    \bottomrule
  \end{tabular}
  }
\end{table}

%% file: appendices/F3.tex
\subsection{Evaluation under standard \textit{n}-way \textit{k}-shot setting}
\label{appendix:n-wayk-shot}

To examine the generality of our approach, we further evaluate it under the standard \textit{n}-way \textit{k}-shot setting. To the best of our knowledge, this setup has not been explicitly explored for general hypergraphs. Accordingly, we designed a new evaluation protocol by drawing inspiration from the experimental setup of Meta-GNN \cite{10.1145/3357384.3358106} and adopting hyperparameter configurations aligned with those used in ED-HNN ~\cite{wang2023equivarianthypergraphdiffusionneural}. The results, summarized in Table~\ref{tab:fewshot_citation}, demonstrate the strong performance of \method in this setting, underscoring its effectiveness for few-shot learning on hypergraphs. Notably, \method achieves competitive results even when applied solely to test episodes, revealing a promising and underexplored direction for future research.

\input{tables/acc_nwaykshot}

%% file: tables/acc_nwaykshot.tex
\begin{table}[htp!]
  \caption{Classification accuracy (\%) under the standard \textit{n}-way \textit{k}-shot setting. Results are averaged over 10 runs. \method achieves superior performance across all configurations.}
  \label{tab:fewshot_citation}
  \centering
  \resizebox{\textwidth}{!}{
  \begin{tabular}{l|cccc}
    \toprule
    \textbf{Methods} & \textbf{Cora (2-way 3-shot)} & \textbf{Cora (2-way 1-shot)} & \textbf{Citeseer (2-way 3-shot)} & \textbf{Citeseer (2-way 1-shot)} \\
    \midrule
    \midrule
    HGNN & $64.4_{\pm 0.1}$ & $55.9_{\pm 0.1}$ & $60.1_{\pm 0.2}$ & $54.9_{\pm 0.2}$ \\
    \addlinespace
    UniGCNII & $68.0_{\pm 0.1}$ & $58.9_{\pm 0.2}$ & $65.2_{\pm 0.1}$ & $56.8_{\pm 1.6}$ \\
    \addlinespace
    AllDeepSets & $52.3_{\pm 0.2}$ & $48.7_{\pm 0.2}$ & $51.1_{\pm 0.2}$ & $50.2_{\pm 0.2}$ \\
    \midrule
    \midrule
    \textbf{ZEN (proposed)} & $\mathbf{73.3_{\pm 0.1}}$ & $\mathbf{62.6_{\pm 0.1}}$ & $\mathbf{71.5_{\pm 0.1}}$ & $\mathbf{62.3_{\pm 0.1}}$ \\
    \bottomrule
  \end{tabular}
  }
\end{table}

%% file: appendices/G.tex
\section{Running times}
The actual running times are reported in Table~\ref{tab:real-time}, with all values measured in seconds. \method runs in less than a second across all datasets, achieving as fast as 0.003s on the Senate dataset.

\input{tables/time_real}

%% file: tables/time_real.tex
\begin{table}[htp!]
\caption{The actual running time of \method and the baseline models, including both training and inference. Each time is measured in seconds. \method exhibits a significant speed advantage over all baselines.}
\label{tab:real-time}
\centering
\resizebox{\textwidth}{!}{
\begin{tabular}{l|cccccccccc}
\toprule
\textbf{Methods} & \textbf{Cora} & \textbf{Citeseer} & \textbf{Pubmed} & \textbf{Cora\_CA} & \textbf{20News} & \textbf{MN40} & \textbf{Congress} & \textbf{Walmart} & \textbf{Senate} & \textbf{House} \\
\midrule \midrule
HGNN & $2.301$ & $2.800$ & $2.423$ & $2.506$ & $2.499$ & $3.958$ & $4.950$ & $20.277$ & $2.540$ & $2.543$ \\
\addlinespace
HNHN & $2.023$ & $1.996$ & $2.877$ & $2.024$ & $3.676$ & $2.720$ & $3.925$ & $14.435$ & $2.278$ & $2.267$ \\
\addlinespace
HCHA & $3.290$ & $3.485$ & $4.448$ & $2.850$ & $2.894$ & $4.072$ & $6.675$ & $8.558$ & $3.297$ & $4.584$ \\
\addlinespace
UniGCN & $4.553$ & $1.932$ & $6.693$ & $2.511$ & $6.825$ & $4.315$ & $8.497$ & $22.121$ & $2.342$ & $1.920$ \\
\addlinespace
UniGCNII & $4.193$ & $3.812$ & $2.069$ & $4.826$ & $4.025$ & $3.831$ & $7.472$ & $24.452$ & $2.277$ & $2.419$ \\
\addlinespace
AllDeepSets & $15.456$ & $18.331$ & $8.889$ & $16.108$ & $8.420$ & $9.378$ & $25.499$ & $60.619$ & $13.232$ & $11.461$ \\
\addlinespace
AllSetTransformer & $2.597$ & $3.167$ & $8.865$ & $3.462$ & $14.273$ & $4.924$ & $6.114$ & $51.702$ & $3.260$ & $3.435$ \\
\addlinespace
ED-HNN & $4.267$ & $4.458$ & $4.407$ & $6.588$ & $11.072$ & $5.664$ & $39.791$ & $31.137$ & $2.336$ & $2.487$ \\
\midrule \midrule
\textbf{ZEN (proposed)} & $\mathbf{0.266}$ & $\mathbf{0.736}$ & $\mathbf{0.807}$ & $\mathbf{0.280}$ & $\mathbf{0.237}$ & $\mathbf{0.197}$ & $\mathbf{0.093}$ & $\mathbf{0.672}$ & $\mathbf{0.003}$ & $\mathbf{0.007}$ \\
\bottomrule
\end{tabular}
}
\end{table}